\documentclass{article}

\usepackage{fullpage,times}

\usepackage{amsthm,amsfonts,amsmath,amssymb,epsfig,color,float,graphicx,verbatim}
\usepackage{algorithm,algorithmic}

\newtheorem{theorem}{Theorem}

\newtheorem{lemma}{Lemma}

\newcommand{\reals}{\mathbb{R}}
\newcommand{\E}{\mathbb{E}}

\newcommand{\ba}{\mathbf{a}}
\newcommand{\be}{\mathbf{e}}
\newcommand{\bx}{\mathbf{x}}
\newcommand{\bw}{\mathbf{w}}

\newcommand{\bv}{\mathbf{v}}

\newcommand{\bsigma}{\boldsymbol{\sigma}}

\newcommand{\Ocal}{\mathcal{O}}

\newcommand{\Dcal}{\mathcal{D}}

\newcommand{\Ncal}{\mathcal{N}}

\newcommand{\Wcal}{\mathcal{W}}

\newcommand{\norm}[1]{\left\|#1\right\|}
\newcommand{\inner}[1]{\left\langle#1\right\rangle}

\newcommand{\secref}[1]{Sec.~\ref{#1}}

\renewcommand{\eqref}[1]{Eq.~(\ref{#1})}
\newcommand{\lemref}[1]{Lemma~\ref{#1}}

\newcommand{\thmref}[1]{Thm.~\ref{#1}}

\title{On the Complexity of Bandit Linear Optimization}
\author{Ohad Shamir\\Weizmann Institute of Science\\\texttt{ohad.shamir@weizmann.ac.il}}
\date{}

\begin{document}

\maketitle

\begin{abstract}
We study the attainable regret for online linear optimization problems with
bandit feedback, where unlike the full-information setting, the player can
only observe its own loss rather than the full loss vector. We show that
the price of bandit information in this setting can be as large as $d$,
disproving the well-known conjecture \cite{dani2007price} that the regret
for bandit linear optimization is at most $\sqrt{d}$ times the
full-information regret. Surprisingly, this is shown using ``trivial''
modifications of standard domains, which have no effect in the
full-information setting. This and other results we present highlight some
interesting differences between full-information and bandit learning, which
were not considered in previous literature.
\end{abstract}

\section{Introduction}

We consider the problem of bandit linear optimization, which is a repeated
game between a player and an adversary. At each round $t=1,\ldots,T$ the
player chooses a point $\bw_t$ from a compact subset $\Wcal$ of $\reals^d$,
and simultaneously the adversary chooses a loss vector $\bx_t\in \reals^d$
(under some regularity assumptions described in \secref{sec:setting}). The
player incurs a loss $\inner{\bx_t,\bw_t}$, and can observe the loss but not
the loss vector $\bx_t$. The player's goal is to minimize its expected
cumulative loss, $\E\left[\sum_{t=1}^{T}\inner{\bx_t,\bw_t}\right]$, where
the expectation is with respect to the player's and adversary's possible
randomization. The performance of the player is measured in terms of expected
regret, defined as
\begin{equation}\label{eq:regret}
\E\left[\sum_{t=1}^{T}\inner{\bx_t,\bw_t}\right]
-\min_{\bw\in\Wcal}\E\left[\sum_{t=1}^{T}\inner{\bx_t,\bw}\right].
\end{equation}

Bandit optimization has proven to be a useful abstraction of sequential
decision-making problems under uncertainty, such as multi-armed bandits and
online routing (see \cite{BuCe12} for a survey). Moreover, using
online-to-batch conversion techniques, algorithms for this setting can be
readily applied to derivative-free stochastic optimization, where our goal is
to stochastically optimize an unknown function given only noisy views of its
values at various points.

The attainable regret in the bandit setting can be compared to the attainable
regret in the full-information setting, where $\bx_t$ is revealed to the
player after each round. Clearly, since the player receives less information
in the bandit setting, the attainable regret will be larger. This degradation
is known as the ``price of bandit information'' \cite{dani2007price}, and
characterizing it for general domains $\Wcal$ has remained an open problem.
However, the standard conjecture and common wisdom (as articulated in
\cite{dani2007price}) is that for linear optimization, this price is at most
a multiplicative $\sqrt{d}$ factor, where $d$ is the dimension. Indeed, as
far as we know, this holds for all domains that have been previously studied
in the literature:
\begin{itemize}
  \item When the domain $\Wcal$ is the corners of the $d$-dimensional
      simplex (a.k.a. multi-armed bandits setting), the minimax optimal
      regret is $\Theta(\sqrt{dT})$, vs. $\Theta(\sqrt{\log(d)T})$ in the
      full-information setting.
  \item When the domain $\Wcal$ is the boolean hypercube $\{-1,+1\}^d$, the
      minimax optimal regret is $\Theta(d\sqrt{T})$, vs.
      $\Theta(\sqrt{dT})$ in the full-information setting
      \cite{dani2007price,AuBuLu11}.
  \item When the domain $\Wcal$ is the unit Euclidean ball, there is an
      algorithm with regret $\Ocal(\sqrt{dT})$, vs. $\Ocal(\sqrt{T})$ in
      the full-information setting. \cite{BubCesKa12}
  \item There is an $\Omega(d\sqrt{T})$ lower bound for a certain
      non-convex subset of the hypercube (a cartesian product of $d$
      $2$-dimensional spheres) \cite{dani2008stochastic}. The corresponding
      full-information regret bound is $\Ocal(\sqrt{dT})$.
\end{itemize}
We note in passing that there are other partial-information settings where
the situation is different, but these are distinct from the bandit linear
optimization setting we focus here (e.g. non-linear bandit optimization
\cite{Shamir12} or using different information feedback, e.g.
\cite{cesa2012combinatorial}).

The main contribution of this paper is disproving this conjecture, and
showing that the price of bandit information for online linear optimization
can be as large as $d$ rather than $\sqrt{d}$. We do this by proving that the
$\Ocal(\sqrt{dT})$ regret upper bound for the Euclidean ball is surprisingly
brittle, and the attainable regret becomes $\Omega(d\sqrt{T})$ after
``trivial'' changes of the domain which do not matter at all in the
full-information setting. These changes include (1) Shifting the ball away
from the origin, and (2) Taking a simple convex subset of the Euclidean ball
with a ``flat'' boundary, such as a cylinder or a capped Euclidean ball. We
also explain how our techniques can be potentially applicable to other
domains. Since the full-information regret in these cases is
$\Ocal(\sqrt{T})$, this establishes a price of bandit information on the
order of $d$. This gap is tight as worst-case over all domains (or all convex
domains), because for any domain, it is possible to get $\Ocal(d\sqrt{T})$
regret \cite{BubCesKa12,hazan2013volumetric}, and the attainable regret in
the full-information case is generally at least $\Omega(\sqrt{T})$.

We note that our lower bounds hold even against a stochastic adversary, which
chooses loss vectors i.i.d. from a given distribution. In such a setting, any
algorithm attaining $\Ocal(\sqrt{d^\alpha  T})$ regret can be used to find an
$\epsilon$-optimal point after $\Ocal(d^\alpha/\epsilon^2)$ rounds (as
described later). This allows us to re-phrase our result in the following
manner, which might be conceptually interesting: If the price of bandit
information was $\sqrt{d}$, then to find an $\epsilon$-optimal point, we
would need $d$ times more rounds in the bandit setting, compared to the
full-information setting (e.g. $d/\epsilon^2$ vs. $1/\epsilon^2$). This is
intuitively very appealing: Each round, we get to see only ``$1/d$ as much
information'' in the bandit setting (a single scalar compared to a
$d$-dimensional vector), hence we need $d$ times more rounds to get the same
amount of information and return a point of similar accuracy. Unfortunately,
our results show that there are cases where the price of viewing scalars vs.
$d$-dimensional vectors is \emph{quadratic} in $d$ (e.g. $d^2/\epsilon^2$ vs.
$1/\epsilon^2$). So in some sense, the number of rounds required and the
amount of information per round cannot be traded-off without significant
loss.

A second contribution of our paper lies in disproving another common
intuition: Namely, that any regret lower bounds with respect to a given
domain $\Wcal$ automatically extend to a its convex hull
$\text{conv}(\Wcal)$. For example, this has been implicitly used to argue
that the $\Omega(d\sqrt{T})$ regret lower bound for the boolean hypercube
$\{-1,+1\}^d$ in \cite{audibertminimax11} extends to the convex hypercube
$[-1,+1]^d$ \cite{BuCe12,hazan2013volumetric}). Again, this is generally true
for online linear optimization in the full-information setting, since the
optimal points in $\text{conv}(\Wcal)$ lie in $\Wcal$, and the player can
always simulate playing over $\text{conv}(\Wcal)$ via randomization over
$\Wcal$. However, in the bandit setting the change in domain also changes the
information feedback structure in non-intuitive ways. For example, if the
loss vectors are binary and $\Wcal$ is the corners of the simplex (a.k.a.
multi-armed bandits), there is a well-known $\Omega(\sqrt{dT})$ regret lower
bound \cite{auer2002nonstochastic}. However, when we convexify the domain and
take $\Wcal$ to be the simplex, then it is possible to get
$\Ocal(\sqrt{\log(d)T})$ regret - same as in the full-information setting! In
fact, such a result can be shown to hold for most continuous domains. We
emphasize that these regret upper bounds are achieved only for binary (or
finitely-valued) losses, and in a manner which is of little practical value.
However, they do demonstrate that one has to be careful when extending bandit
optimization results from a finite domain to its convex hull, and that loss
vectors supported on a finite set are not appropriate to prove bandit lower
bounds for continuous domains. In any case, for completeness, we provide
minimax optimal regret bounds for the simplex and the hypercube, using the
techniques we develop here. For the hypercube, this bound is
$\Theta(d\sqrt{T})$, and formally establishes that this is the minimax
optimal regret for bandit linear optimization over convex domains.

\subsection{Lower Bound Techniques and Main Ideas}

In this subsection, we informally sketch the technical ideas used to get our
lower bounds. The various domains we discuss below are sketched in Figure
\ref{fig:domains}. The result on continuous domains and finitely-valued
losses is a bit orthogonal, and described separately in \secref{sec:loss}.

\begin{figure}[t]
\centering
\includegraphics[scale=0.4]{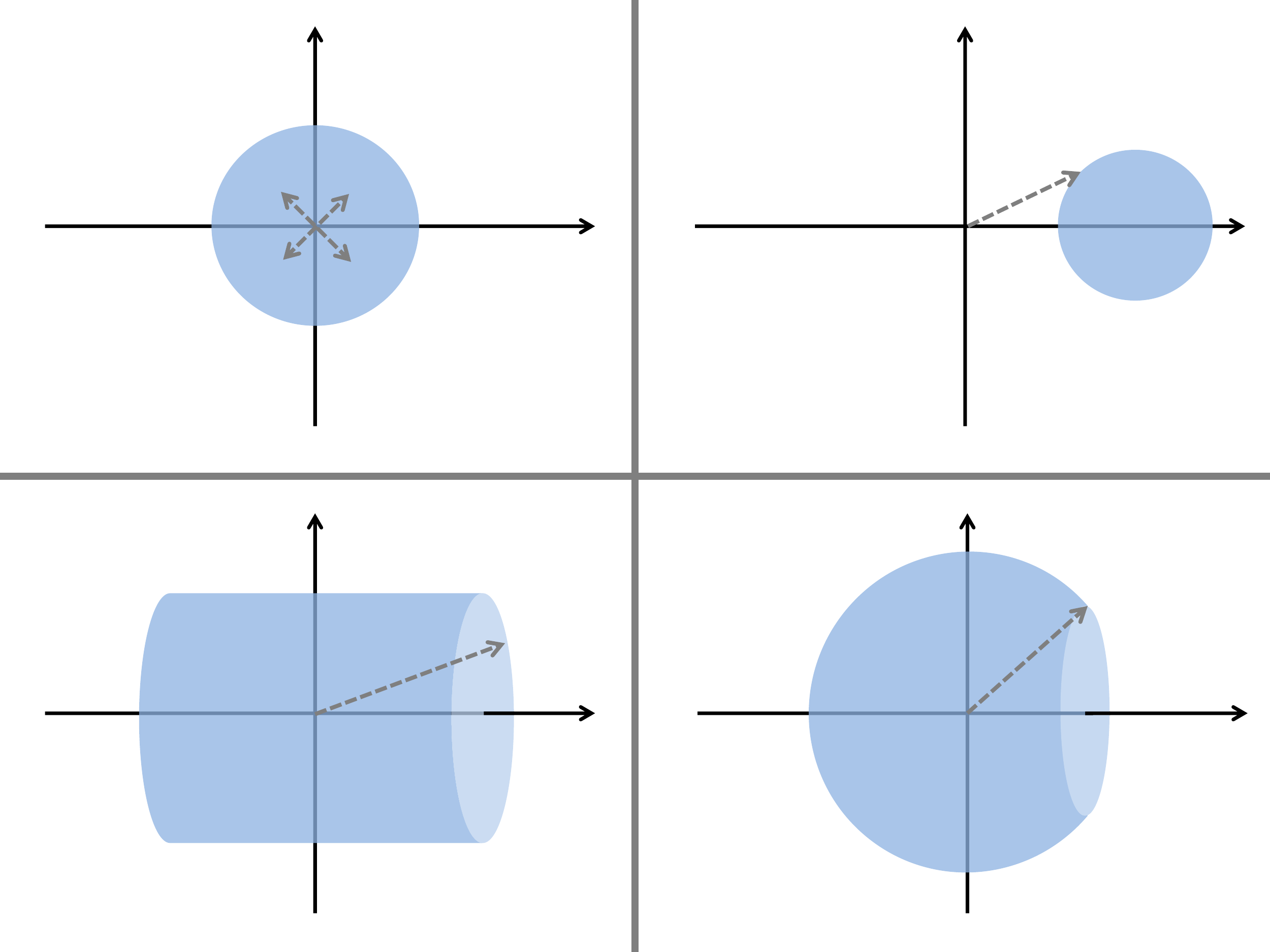}
\caption{Various domains considered in the paper: Origin-centered Euclidean unit ball; Shifted Euclidean unit ball;
cylinder; and capped Euclidean ball. For a domain such as the unit ball, and for minimizing error, we can
get a bounded-variance estimate of the expected loss vector by querying at random from an origin-centered hypercube. However,
this is not possible in the other settings considered here, either because the domain is bounded away from the origin, or because the player
needs to query far and to one side of the origin to attain small regret.}
\label{fig:domains}
\end{figure}

As mentioned earlier, our lower bounds are actually proven for the easier
setting of \emph{stochastic} bandit linear optimization, where the adversary
is constrained to sample each loss vector $\bx_t$ i.i.d. from the same
distribution. Moreover, some of the bounds also apply to the easier goal of
minimizing error rather than regret, which may be of independent interest: In
this case, the player may choose $\bw_1,\ldots,\bw_T$ arbitrarily, observing
$\inner{\bx_1,\bw_1},\ldots,\inner{\bx_T,\bw_T}$, and then needs to return a
vector $\hat{\bw}$ which minimizes the expected error
\begin{equation}\label{eq:error}
\E\left[\inner{\bar{\bx},\hat{\bw}}\right]-\min_{\bw\in\Wcal}\inner{\bar{\bx},\bw}.
\end{equation}
where $\bar{\bx}=\E[\bx]$. This goal is more relevant in a stochastic
optimization setting, where we attempt to optimize a stochastic linear
function based on querying its values $\inner{\bx_t,\bw_t}$ at various
points. It is easy to show that any algorithm, attaining expected regret of
$R$ after $T$ rounds, can attain an $R/T$ error by returning the vector
$\bar{\bw}=\frac{1}{T}\sum_{t=1}^{T}\bw_t$. Thus, any lower bound for
minimizing error in a stochastic setting implies the same lower bound on
minimizing regret and a possible non-stochastic setting, times a $T$ factor.
For example, a $\sqrt{d/T}$ error lower bound implies a $\sqrt{dT}$ regret
lower bound.

For the goal of minimizing error in a stochastic setting, it turns out that
there is a simple strategy attaining $\sqrt{d/T}$ error under appropriate
conditions: In each round, we randomly sample from some hypercube centered
around the origin, and use it to create an unbiased estimate of the loss
vector, with variance $\Ocal(d)$. Repeating this for $T$ rounds and
averaging, we get an estimate $\tilde{\bx}$ of the expected loss vector
$\bar{\bx}$, which is unbiased and with variance $\Ocal(d/T)$. We then return
the point $\arg\min_{\bw\in\Wcal} \inner{\tilde{\bx},\bw}$, leading to an
$\Ocal(\sqrt{d/T})$ error bound.

Unfortunately, this method breaks down if we cannot sample from such an
origin-centered hypercube. In particular, if the variance of
$\inner{\bw,\bx}$ for any $\bw$ in the domain is lower bounded by a constant
(independent of the dimension), then the variance of the estimator becomes
$\Ocal(d^2)$ rather than $\Ocal(d)$. Clearly, this is not true when $\bw$ can
be arbitrarily close to the origin: If $\bw_t=\mathbf{0}$, then the variance
of $\inner{\bw,\bx}$ is always zero. However, it can happen when the domain
is bounded away from the origin. Using this key observation together with
canonical information-theoretic tools (i.e. reduction to a hypothesis testing
problem), we can show a $\Omega(d/\sqrt{T})$ error lower bound for a simple
domain bounded away from the origin, such as a shifted Euclidean ball. Due to
the relationship between error and regret, this leads to an
$\Omega(d\sqrt{T})$ regret lower bound for such domains.

Another way to force $\inner{\bw,\bx}$ to have constant variance, even if the
domain contains the origin, is when the player must pick points far away from
the origin most of the time. This cannot be enforced when the goal is
minimizing error. However, when the goal is minimizing regret, then we can
construct a situation where any point close to the origin leads to a large
regret. In that case, either the player picks points close to the origin, and
gets large regret, or picks points far from the origin, leading to the
variance of $\inner{\bw_t,\bx_t}$ being large. This is a lose-lose situation,
and carefully formalizing it leads to a $\Omega(d\sqrt{T})$ regret lower
bound, similar to the shifted Euclidean ball setting. We formally show this
for a cylindrical domain (see figure \ref{fig:domains}) as well as a
hypercube, but the proof technique appears applicable to other domains, such
as a capped Euclidean ball, or more generally any domain with a flat surface
orthogonal to the origin. In all these cases, the adversary strategy is to
choose an expected loss vector at random from the flat surface, and introduce
a constant amount of stochastic noise in the orthogonal dimension. Since the
flat surface is far from the origin, the player must deal with constant
variance if it wishes its regret to be small.

Although our lower bounds are shown for particular domains, we believe that
the properties we identified -- distance from origin, variance of the losses,
and flatness of parts of the domain boundary -- can play a key role in
characterizing the attainable regret for any given domain.

\section{Preliminaries}\label{sec:setting}

We use bold-faced letters to denote vectors (e.g. $\bx=(x_1,\ldots,x_d)$). We
let $\norm{\cdot}$ denote the Euclidean norm, and $\norm{\cdot}_p$ the
$p$-norm. We also define the function $\norm{\cdot}_*:\reals^d\mapsto
[0,\infty)$ as
\[
\norm{\bx}_{*} = \max_{\bw\in\Wcal}|\inner{\bw,\bx}|.
\]
When $\Wcal$ is a symmetric convex set, then $\norm{\cdot}_{*}$ is the dual
norm to the norm whose unit ball is $\Wcal$, hence the notation. However, we
will use this notation even when $\Wcal$ is not convex and symmetric. We also
define $\Ncal(\ba,\Sigma)$ to be the multivariate Gaussian distribution with
mean $\ba$ and covariance matrix $\Sigma$, and let $I$ denote the identity
matrix.

As discussed in the introduction, our results focus on bandit linear
optimization in a stochastic setting, where the loss vectors are assumed to
be drawn i.i.d. from an unknown distribution $\Dcal$ with mean $\bar{\bx}$.
In this setting, the expected regret (as defined in \eqref{eq:regret}) can be
equivalently written as
\begin{equation}\label{eq:stochregret}
\E\left[\inner{\bar{\bx}~,~\sum_{t=1}^{T}\bw_t}\right]-T\min_{\bw\in\Wcal}\inner{\bar{\bx},\bw}.
\end{equation}
Clearly, if any distribution is allowed, then the inner products above can be
unboundedly large, and no interesting regret bound is possible. To prevent
this, it is standard in the literature to assume that $\inner{\bw,\bx}$ is
essentially bounded. Formally, we assume that $\Dcal$ satisfies the following
two conditions:
\begin{itemize}
  \item $\norm{\bar{\bx}}_{*}\leq 1$
  \item $\Pr_{\bx\sim\Dcal}\left(\norm{\bx}_{*}>z\right)\leq 2\exp(-z^2/2)$
      for all $z\geq 1$.
\end{itemize}
We denote any such distribution as a ``valid'' distribution. The latter
condition requires $\norm{\bx}_{*}$ to have sub-Gaussian tails (or more
precisely, tails dominated by a standard Gaussian random variable), and is a
slight relaxation of the standard `dual' setting (e.g.
\cite{BubCesKa12,audibertminimax11}), where it is assumed that
$\norm{\bx}_{*}\leq 1$ with probability $1$. We choose this purely for
technical convenience, since it allows us to use Gaussian distributions for
$\Dcal$, and does not materially affect algorithmic approaches we're aware
of. Moreover, in terms of lower bounds, we lose almost nothing by this
relaxation: Any lower bound for our setting can be transformed into an
equivalent lower bound in the standard dual setting (where $\norm{\bx}_*\leq
1$ with probability $1$), at the cost of a $\sqrt{\log(T)}$ factor. See
Appendix \ref{app:bounded} for details.

\section{Expected Error}

We begin by considering the attainable performance in terms of expected error
(\eqref{eq:error}), recalling that error lower bounds immediately transfer to
regret lower bounds.

To motivate our results, let us show that it is easy to obtain
$\Ocal(\sqrt{d/T})$ error upper bounds under a mild condition. To do so,
suppose that $\Wcal$ contains the set $\{-\mu,\mu\}^d$ for some $\mu>0$, and
consider the following simple player strategy: In each round $t$, the player
draws a vector $\bsigma\in\{-1,+1\}^d$ uniformly at random, computes
$v=\inner{\bx_t,\mu\bsigma}$ (possible since $\mu\bsigma\in\Wcal$), and
computes the estimator $\tilde{\bx}_t = \frac{1}{\mu}v\bsigma =
\inner{\bx,\bsigma}\bsigma$. It is easy to verify that
\begin{equation}\label{eq:condunb}
\E[\tilde{\bx}_t|\bx_t] = \bx_t~~,~~ \E[\norm{\tilde{\bx}_t}_2^2|\bx_t] = d\norm{\bx_t}_2^2,
\end{equation}
so $\tilde{\bx}_t$ is an unbiased estimator of $\bx_t$ with variance bounded
by $\Ocal(d)$. After $T$ rounds, the player computes
$\tilde{\bx}=\frac{1}{T}\sum_{t=1}^{T}\tilde{\bx}_t$, and returns
$\hat{\bw}=\arg\min_{\bw\in\Wcal}\inner{\tilde{\bx},\bw}$.

\begin{theorem}\label{thm:errorup}
  Suppose that $\Wcal$ is a subset of the unit Euclidean ball $\{\bw:\norm{\bw}_2\leq 1\}$,
  and contains $\{-\mu,\mu\}^d$ for some $\mu>0$. Then for any distribution over loss vectors $\bx$,
  such that $\E[\norm{\bx}_2^2]\leq c^2$ and with mean $\bar{\bx}$, the player strategy described above satisfies
  \[
  \E\left[\inner{\bar{\bx},\hat{\bw}}\right]-\min_{\bw\in\Wcal}\inner{\bar{\bx},\bw} \leq 2c\sqrt{\frac{d}{T}}.
  \]
\end{theorem}
We note that \cite{BubCesKa12} provides a more sophisticated algorithm with
$\Ocal(\sqrt{dT})$ regret, even against a non-stochastic adversary.

\begin{proof}
The idea of the proof is that since $\{\tilde{\bx}_t\}_{t=1}^{T}$ are
unbiased estimate of $\bx_t$, which are themselves random variables with mean
$\bar{\bx}$, then their average $\tilde{\bx}$ will be a good estimate of
$\bar{\bx}$, hence minimizing $\inner{\tilde{\bx},\bw}$ will approximately
minimize $\inner{\bar{\bx},\bw}$.

More formally, let $\bw^*=\arg\min_{\bw\in\Wcal}\inner{\bar{\bx},\bw}$. Also,
recall that any $\bw\in\Wcal$ satisfies $\norm{\bw}_2\leq 1$, and therefore
$\inner{\bx,\bw}-\inner{\bx',\bw}\leq \norm{\bx-\bx'}_2$ for all
$\bw\in\Wcal$ by the Cauchy-Schwartz inequality. Finally, recall that
$\hat{\bw}=\arg\min_{\bw\in\Wcal}\inner{\tilde{\bx},\bw}$. Using these
observations, we have
\[
\inner{\bar{\bx},\hat{\bw}}
\leq \inner{\tilde{\bx},\hat{\bw}}+\norm{\bar{\bx}-\tilde{\bx}}
\leq \inner{\tilde{\bx},\bw^*}+\norm{\bar{\bx}-\tilde{\bx}}
\leq \inner{\bar{\bx},\bw^*}+2\norm{\bar{\bx}-\tilde{\bx}},
\]
and therefore
\begin{equation}\label{eq:xdiff}
\E\left[\inner{\bar{\bx},\hat{\bw}}\right]-\min_{\bw\in\Wcal}\inner{\bar{\bx},\bw}
~=~
\E\left[\inner{\bar{\bx},\hat{\bw}}\right]-\inner{\bar{\bx},\bw^*}
~\leq~
2\E\left[\norm{\bar{\bx}-\tilde{\bx}}\right]
~\leq~
2\sqrt{\E\left[\norm{\bar{\bx}-\tilde{\bx}}^2\right]},
\end{equation}
where the last step is by Jensen's inequality. Recalling the definition of
$\tilde{\bx}$, this equals
\[
2\sqrt{\E\left[\left(\frac{1}{T}\sum_{t=1}^{T}\left(\tilde{\bx}_t-\bar{\bx}\right)\right)^2\right]}.
\]
According to equation \eqref{eq:condunb}, $\E[\tilde{\bx}_t|\bx_t]=\bx_t$,
but we also have $\E[\bx_t]=\bar{\bx}$, and therefore
$\E[\tilde{\bx}_t-\bar{\bx}]=\bar{\bx}-\bar{\bx}=0$, so each summand in the
equation above is zero-mean. Moreover, they are i.i.d. since each
$\tilde{\bx}_t$ is computed based on the independent realization $\bx_t$.
Therefore, the equation above equals
\[
2\sqrt{\frac{1}{T}\E\left[\left(\tilde{\bx}_1-\bar{\bx}\right)^2\right]}
~\leq~
2\sqrt{\frac{1}{T}\E\left[\norm{\tilde{\bx}_1}^2\right]},
\]
which by \eqref{eq:condunb} is at most
$2\sqrt{\frac{1}{T}d~\E[\norm{\bx_1}_2^2]} \leq 2c\sqrt{d/T}$ as required.
\end{proof}

For example, a special case of the theorem above implies an
$\Ocal(\sqrt{d/T})$ error upper bound for the origin-centered unit Euclidean
ball. The key property used was the ability to query at a random point in
$\{-\mu,+\mu\}^d$ (a scaled origin-centered Boolean hypercube), which allowed
computing an unbiased estimator of each $\bx_t$ with variance $\Ocal(d)$.

This observation leads us to guess that when we cannot query from such a set,
the learning problem may be harder. One case in which querying such a set is
impossible is when $\Wcal$ is convex and bounded away from the origin. In a
full-information learning setting, shifting $\Wcal$ away from the origin
doesn't increase the learning complexity in general. Surprisingly, in the
bandit setting this turns out to make a huge difference: Even for the unit
Euclidean ball, shifting it so it doesn't include the origin is sufficient to
make the attainable error jump from $\sqrt{d/T}$ to $d/\sqrt{T}$:

\begin{theorem}\label{thm:unorigin}
Suppose that $d>1$, and let $\Wcal=\{\ba+\bw:\norm{\bw}_2\leq 1\}$, where
$\ba=(2,0,\ldots,0)\in \reals^d$. Then for any player strategy returning
$\hat{\bw}\in\Wcal$, there exists a valid distribution over loss vectors with
mean $\bar{\bx}$ such that
  \[
  \E[\inner{\bar{\bx},\bw}]-\min_{\bw\in\Wcal}\inner{\bar{\bx},\bw} \geq c\min\left\{1~,~\frac{d-1}{\sqrt{T}}\right\},
  \]
  where $c$ is a positive universal constant.
\end{theorem}
The formal proof appears in \secref{sec:proofs}, and we note that due to the
rotational symmetry of our linear optimization setting, the same lower bound
would hold for a Euclidean ball shifted in any other direction. Although the
result is specifically for a shifted Euclidean ball, we conjecture that the
result can be generalized to other domains as well, as long as they are
sufficiently large and at a constant distance from the origin.

Although the notion of a domain not containing the origin may appear unusual
at first, it can actually be necessary to model some common situations. One
example is when the domain has a linear constraint bounding it away from the
origin, such as the probability simplex or a convex hull of points lying in
the same quadrant. As another example, suppose our losses actually take the
form $\inner{\bw_t,\bx_t}+\xi_t$ where $\xi_t$ is some noise or other bias
term not controlled by the player. This is possible to model in the bandit
setting, by adding a dimension and playing over the domain $\Wcal\times
\{1\}$, using loss vectors of the form
$(x_{t,1},x_{t,2},\ldots,x_{t,d},\xi_t)$. But this leads to a domain bounded
away from the origin. In fact, if $\Wcal$ is the origin-centered Euclidean
ball, then the techniques of \thmref{thm:unorigin} readily imply a
$\Omega(d/\sqrt{T})$ error lower bound for the domain $\Wcal\times \{1\}$.

As discussed earlier, the key idea in proving \thmref{thm:unorigin} is that
we can construct a distribution where the variance of $\inner{\bw,\bx}$ for
any $\bw\in\Wcal$ is lower bounded by a positive constant, independent of the
norm of $\bw$ or the dimension. The effect of this is best seen in the
estimation procedure described before \eqref{eq:condunb}: Recall that there
we chose a query point $\bw_t=\mu\bsigma$ where $\bsigma\in \{-1,+1\}^d$, in
which case the conditional second moment of the estimator
$\tilde{\bx}_t=\frac{1}{\mu}v\bsigma$, where $v=\inner{\bx_t,\bw_t}$
,satisfies
\begin{equation}\label{eq:dev}
\E[\norm{\tilde{\bx}_t}_2^2|\bx_t] = \frac{1}{\mu^2}\E[v^2\norm{\bsigma}_2^2]
= \frac{d~\E[v^2]}{\mu^2}.
\end{equation}
Since $v=\inner{\bx_t,\bw_t}=\mu\inner{\bx_t,\bsigma}$, we see that $\E[v^2]$
scales with $\mu^2$, and therefore \eqref{eq:dev} is $\Ocal(d)$, independent
of $\mu$. In contrast, if $\E[v^2]=\E[\inner{\bx_t,\bw_t}^2]$ was forced to
have constant positive variance, then \eqref{eq:dev} is at least
$\Omega(d/\mu^2)$. Moreover, if $\Wcal$ has bounded norm, then $\mu\leq
\Ocal(1/\sqrt{d})$, so \eqref{eq:dev} would scale as $\Ocal(d^2)$ rather than
$\Ocal(d)$, eventually leading to an error bound scaling as $d/\sqrt{T}$
rather than $\sqrt{d/T}$.

\section{From Error to Regret}

Having considered domains bounded away from the origin, it is natural to ask
whether this is the only condition leading to $\Omega(d\sqrt{T})$ regret
lower bounds. If the domain does contain the origin, then because of
\thmref{thm:errorup}, it seems unlikely to show such lower bounds by proving
error lower bounds. However, we can exploit the fact that attaining small
regret is harder than attaining small error. In this section, we show how in
fact it can be strictly harder: We identify a situation which leads to
$\Omega(d\sqrt{T})$ regret lower bounds, even if the domain contains the
origin, and even though better error upper bounds are possible.

Specifically, the following theorem demonstrates such a lower bound for a
domain consisting of an origin-centered cylinder.
\begin{theorem}\label{thm:regretorigin}
Suppose that $d>1$, and let $\Wcal=[-1,1]\times
\{\bw\in\reals^{d-1}:\norm{\bw}_2\leq 1\}$. Then for any player strategy,
there exists a valid distribution over loss vectors with mean $\bar{\bx}$
such that
  \[
  \E\left[\sum_{t=1}^{T}\inner{\bar{\bx},\bw_t}\right]-\min_{\bw\in\Wcal}\sum_{t=1}^{T}\inner{\bar{\bx},\bw} \geq c (d-1)\sqrt{T}
  \]
  for any $T\geq d^4/16$, where $c$ is a positive universal constant.
\end{theorem}
Note that the bound is a bit weaker than \thmref{thm:unorigin}, in that it
only holds for sufficiently large $T$. However, since this is a lower bound,
it is still sufficient for proving that no algorithm will attain
$o(d\sqrt{T})$ regret in general here.

The proof relies on the following construction: The adversary chooses a
vector $\bsigma\in \{-1,+1\}^d$ uniformly at random, and constructs a loss
vector distribution, whose expectation $\bar{\bx}$ is
$\left(-\frac{1}{4},\mu\sigma_1,\ldots,\mu\sigma_d\right)$ (with $\mu$ being
a small scaling factor, on the order of $\sqrt{d/T}$ if $T$ is sufficiently
large), and where the first coordinate has constant Gaussian noise. It is
possible to show that to get small regret, the player essentially needs to
identify $\sigma_1,\ldots,\sigma_d$. This would have been possible if the
player queried at points $\bw_t$ whose first coordinate is $0$. However,
since $\bar{\bx}$ has a large negative weight on the first coordinate, then
$w_{t,1}$ must be large to get small regret. But because of the stochastic
noise in the first coordinate, this means that $\inner{\bw_t,\bx_t}$ will
have large variance. A formal analysis of this leads to a regret lower bound
of the form
\[
\Omega\left(\sum_{t=1}^{T}\E[1-w_{t,1}]+\mu\sqrt{d}T\left(1-\mu\sqrt{\frac{1}{d}\sum_{t=1}^{T}\E\left[\frac{1}{w_{t,1}^2+\frac{1}{d}}\right]}\right)\right).
\]
The idea now is that no matter how the first coordinates
$\{w_{t,1}\}_{t=1}^{T}$ are chosen by the player, the regret will be large:
If they are significantly smaller than $1$, then the first term above will be
on the order of $T$. But if $w_{t,1}$ is almost $1$, then the square root
term will be small, again leading to a large regret. A careful analysis shows
that the regret will always be $\Omega(d\sqrt{T})$.

Although a precise characterization is non-trivial, the proof technique
appears to be potentially applicable to any domain with a flat surface, which
is sufficiently large to contain the points
$\left(-a,\mu\sigma_1,\ldots,\mu\sigma_d\right)$ for some constant $a\in
(0,1)$ and suitable scaling factor $\mu$ (possibly after rotating the domain
around the origin). For example, the same proof technique would apply to a
capped Euclidean ball, $\Wcal=\left\{\bw:\norm{\bw}\leq 1, w_1\leq c\right\}$
for some constant $c\in (0,1)$. The flatness of the surface seems to be
important. To see this, let us understand why the proof breaks down, when
instead of a cylinder or a capped ball, our domain is the origin-centered
Euclidean unit ball (for which we know that we can get $\Ocal(\sqrt{dT})$
regret \cite{BubCesKa12}). First, when $T$ is sufficiently large compared to
$d$, we choose $\mu$ on the order of $\sqrt{d/T})$ (this is the regime which
makes detecting $\sigma_1,\ldots,\sigma_d$ information-theoretically hard).
In that case, we have
$\bar{\bx}=\left(-a,\sqrt{\frac{d}{T}}\sigma_1,\ldots,\sqrt{\frac{d}{T}}
\sigma_d\right)$, and the optimal play is
$\bw^*=-\frac{1}{\norm{\bar{\bx}}}\bar{\bx}=
\frac{1}{\sqrt{a^2+\frac{d^2}{T}}}
\left(a,-\sqrt{\frac{d}{T}}\sigma_1,\ldots,-\sqrt{\frac{d}{T}}\sigma_d\right)$.
However, with an origin-centered ball, the player doesn't need to detect
$\sigma_1,\ldots,\sigma_d$ to get small regret: The player can go ``further''
along the first coordinate, and just play the fixed point $\bw =
(1,0,\ldots,0)$ every round. The total expected regret is then
\begin{align*}
&T\left(\inner{\bar{\bx},\bw}-\inner{\bar{\bx},\bw^*}\right)
~=~ T\left(\bar{x}_1+\frac{1}{\norm{\bar{\bx}}}\norm{\bar{\bx}}^2\right)
~=~ T\left(\bar{x}_1+\norm{\bar{\bx}}\right)\\
&=~ T\left(-a+\sqrt{a^2+\frac{d^2}{T}}\right)
~\leq~ T\left(-a+a+\frac{1}{2a}\frac{d^2}{T}\right) ~=~ \frac{1}{2a}d^2 ~=~ \Ocal(d^2),
\end{align*}
where we used the fact (immediate by Taylor expansion) that $\sqrt{a^2+x}\leq
a+\frac{1}{2a}x$ for all $x\geq 0$. So, we see that instead of
$\Omega(d\sqrt{T})$ regret as in the case of a cylinder or capped ball, here
the player can achieve a much smaller $\Ocal(d^2)$ regret (assuming $T$ is
sufficiently large). In fact, this $\Ocal(d^2)$ regret bound is tight for our
construction -- carrying through the lower bound derivation as in
\thmref{thm:regretorigin}, but for the origin-centered ball, yields an
$\Omega(d^2)$ regret lower bound for sufficiently large $T$.

\section{Loss Vectors from a Finite Set}\label{sec:loss}

Many of the regret lower bounds shown in the literature for finite domains
(such as for the corners of the simplex and for the boolean hypercube
\cite{auer2002nonstochastic,dani2007price,audibertminimax11}) use loss
vectors from a finite set (e.g. where each entry is binary). Based on
analogues to the full information setting, it is often argued that these
lower bound constructions also extend to the (continuous) convex hull of
these domains. In this section, we point out that these analogues can be
dangerous and generally do not hold.

In particular, the theorem below demonstrates a possibly surprising fact: If
we perform bandit linear optimization over a continuous domain, and the loss
vectors come from a finite set, then we can actually get the same regret as
in the full information case -- generally much smaller than the lower bound
one would hope to achieve. This also holds for non-stochastic adversaries. We
hasten to emphasize that the way we achieve this is ``cheating'' and not very
useful in practice, since it heavily relies on the ability to perform
arbitrary-precision computations. Nevertheless, it demonstrates that one has
to be careful when extending bandit optimization results from a finite domain
to its convex hull. Moreover, it shows that loss vectors supported on a
finite set are not suitable to prove bandit lower bounds for continuous
domains.

For simplicity, we will prove the result for the simplex and for
binary-valued loss vectors, but from the proof it is easily seen to be
extendable to other continuous domains in general (such as the hypercube
$[-1,+1]^d$), as well as the loss vectors coming from any finite set.

\begin{theorem}
Suppose $\Wcal=\{\bw:\sum_i w_i=1, \forall i~w_i\geq 0\}$ is the probability
simplex in $\reals^d$, and suppose the loss vectors are from the set
$\{0,1\}^d$. Then there exist a deterministic player strategy in the bandit
setting, such that for any adversarial strategy for choosing loss vectors
$\bx_1,\ldots,\bx_T$,
\[
\sum_{t=1}^{T}\inner{\bx_t,\bw_t}-\min_{\bw\in\Wcal}\sum_{t=1}^{T}\inner{\bx_t,\bw}\leq
\Ocal\left(\sqrt{\log(d)T}\right)
\]
\end{theorem}
\begin{proof}
  The proof uses the following observation: In the setting described
  above, it is possible for the player to determine $\bx_t$ precisely by
  perturbing its chosen point by an arbitrarily small amount, and feed it into a deterministic full-information algorithm
  for this domain, such as hedge \cite{FrSc97,CesLu06}. Such algorithms attain
  $\Ocal(\sqrt{\log(d)T})$ regret, from which the result follows.
  
  More precisely, suppose that at the beginning of round $t$, the full-information algorithm
  determines that one should play point $\hat{\bw}_t$. Let $p=\lceil\log_{10}(T)\rceil$, and define
  \[
  \bw'_t = \text{clip}_p(\hat{\bw}_t)+(10^{-p-1},10^{-p-2},\ldots,10^{-p-d}),
  \]
  where $\text{clip}_p(\bw)$ clips every entry of $\bw$ (in its decimal representation) to $p$ places after
  the decimal point. Note that for any $\bx\in\{0,1\}^d$, if we get
  $\inner{\bx,\bw'_t}=\sum_{i=1}^{d}x_i w'_{t,i}$, then we can determine
  $\bx$ precisely: We just need to look at the digits in locations $p+1,p+2,\ldots,p+d$ after the decimal point.
  A small technical issue is that $\bw'_t$ does not lie in the simplex $\Wcal$, so it cannot be chosen by the player.
  To fix this, the player picks $\bw_t= \frac{1}{\norm{\bw'_t}_1}\bw'_t$, which
  indeed lies in the simplex. Given the loss $\inner{\bx_t,\bw_t}$, we just
  multiply it by the known quantity $\norm{\bw'_t}_1$ to get
  $\inner{\bx_t,\bw'_t}$, from which we can read off $\bx_t$ and feed it back to
  the full-information algorithm.
  
  By standard regret guarantees (e.g. \cite{FrSc97,CesLu06})\footnote{Strictly speaking, these
  algorithms are phrased so as to minimize $\max_{\bw\in\Wcal}\sum_{t=1}^{T}\inner{\bx_t,\bw}-\sum_{t=1}^{T}\inner{\bx_t,\bw_t}$, but
  they can be easily converted to our notion of regret by feeding them with $\mathbf{1}-\bx_t$ where $\mathbf{1}$ is the all-ones vector.}, we therefore get that
  \begin{equation}\label{eq:regrethedge}
  \sum_{t=1}^{T}\inner{\bx_t,\hat{\bw}_t}-\min_{\bw\in\Wcal}\sum_{t=1}^{T}\inner{\bx_t,\bw}\leq \Ocal\left(\sqrt{\log(d)T}\right)
  \end{equation}
  To get a regret bound using the actual plays
  $\bw_t$, we note that we can lower bound the left hand side of \eqref{eq:regrethedge} by
  \begin{equation}\label{eq:perturb}
  \sum_{t=1}^{T}\inner{\bx_t,\bw_t}-\min_{\bw\in\Wcal}\sum_{t=1}^{T}\inner{\bx_t,\bw}-
  \sum_{t=1}^{T}|\inner{\bx_t,\hat{\bw}_t-\bw_t}|,
  \end{equation}
  and that
  \begin{align*}
  |\inner{\bx_t,\hat{\bw}_t-\bw_t}| &~\leq~ \norm{\hat{\bw}_t-\bw_t}_1
  ~=~ \norm{\frac{1}{\norm{\bw'_t}_1}\bw'_t-\bw_t}
  ~\leq~ \norm{\bw'_t-\bw_t}_1+\norm{\bw'_t-\frac{1}{\norm{\bw'_t}_1}\bw'_t}_1\\
  &~=~\norm{\bw'_t-\bw_t}_1+\left|1-\frac{1}{\norm{\bw'_t}_1}\right|\norm{\bw'_t}_1
  ~=~\norm{\bw'_t-\bw_t}_1+\left|\norm{\bw'_t}_1-1\right|.
  \end{align*}
  By definition of $\bw'_t$ and the fact that $\bw_t$ is on the simplex,
  it is easy to verify that $\norm{\bw'_t}_1\in \left[1-\frac{1}{T},1+\frac{1}{T}\right]$ and that
  $\norm{\bw'_t-\bw_t}_1\leq \frac{1}{T}$, so the above equals at most $2/T$,
  and we can lower bound \eqref{eq:perturb} by
  \[
  \sum_{t=1}^{T}\inner{\bx_t,\hat{\bw}_t}-\min_{\bw\in\Wcal}\sum_{t=1}^{T}\inner{\bx_t,\bw}-2
  \]
  Combining this with \eqref{eq:regrethedge}, we get that
  \[
\sum_{t=1}^{T}\inner{\bx_t,\bw_t}-\min_{\bw\in\Wcal}\sum_{t=1}^{T}\inner{\bx_t,\bw}\leq \Ocal(\sqrt{\log(d)T})+2 = \Ocal(\sqrt{\log(d)T}).
  \]
  as required.
\end{proof}

In light of this type of result, we are not aware of explicit minimax regret
bounds for the simplex (or unit $1$-norm ball) and the hypercube in the
literature. For completeness, we provide such bounds below. Besides relying
on continuous-valued rewards, the lower bounds require the techniques of
\thmref{thm:unorigin} and \thmref{thm:regretorigin}: For the simplex, we
provide an error lower bound relying on the fact that the simplex is bounded
away from the origin; Whereas for the hypercube, we rely on it having a flat
surface orthogonal to the origin. The constructions we use are slightly
different though, due to the different shapes of the domains.

\begin{theorem}\label{thm:simplexunorigin}
Suppose that $d>1$, and let $\Wcal$ be either the simplex $\{\bw:\forall
i~w_i\geq 0~,~ \norm{\bw}_1=1\}$ or the unit $1$-norm ball
$\{\bw:\norm{\bw}_1\leq 1\}$. Then for any player strategy returning some
$\hat{\bw}\in \Wcal$, there exists a valid distribution over loss vectors
with mean $\bar{\bx}$ such that
  \[
  \E[\inner{\bar{\bx},\bw}]-\min_{\bw\in\Wcal}\inner{\bar{\bx},\bw} \geq c\min\left\{1~,~\sqrt{\frac{d}{T}}\right\},
  \]
  where $c$ is a positive universal constant.
\end{theorem}
This leads to a $\Omega(\sqrt{dT})$ regret lower bound, which matches the
$\Ocal(\sqrt{dT})$ upper bound for the simplex attained with multi-armed
bandit algorithms such as EXP3 \cite{auer2002nonstochastic}.

\begin{theorem}\label{thm:hypercubeorigin}
Suppose that $d>1$, and let $\Wcal=[-1,1]^d$. Then for any player strategy,
there exists a valid distribution over loss vectors with mean $\bar{\bx}$
such that
  \[
  \E\left[\sum_{t=1}^{T}\inner{\bar{\bx},\bw_t}\right]-\min_{\bw\in\Wcal}\sum_{t=1}^{T}\inner{\bar{\bx},\bw} \geq c (d-1)\sqrt{T}
  \]
  for any $T\geq d^4/4$, where $c$ is a positive universal constant.
\end{theorem}
This $\Omega(d\sqrt{T})$ regret lower bound matches the $\Ocal(d\sqrt{T})$
upper bound attained in \cite{BubCesKa12} (using an algorithm which actually
plays only on the corners of the hypercube).

\section{Proofs}\label{sec:proofs}

\subsection{Proof of \thmref{thm:unorigin}}

To simplify notation, suppose that the game takes place in $\reals^{d+1}$ for
some $d>0$. We denote the first coordinate as coordinate $0$, and the other
coordinates as $1,2,\ldots,d$.

By Yao's minimax principle, it is sufficient to provide a randomized strategy
to choose a loss vector distribution $\Dcal$, such that for any deterministic
player, the expected error is as defined in the theorem. In particular, we
use the following strategy:
\begin{itemize}
  \item Choose $\bsigma \in \{-1,+1\}^d$ uniformly at random.
  \item Use the distribution $\Dcal_{\bsigma}$ over loss vectors $\bx$,
      defined as follows: $(x_1,\ldots,x_d)$ is fixed to be $\mu\bsigma$
      (where $\mu\leq \frac{1}{2\sqrt{d}}$ is a parameter to be chosen
      later), and $x_0$ has a Gaussian distribution
      $\Ncal\left(0,\frac{1}{36}\right)$.
\end{itemize}
First, let us verify that any $\Dcal_{\bsigma}$ is a valid distribution. We
have
\[
\E[\norm{\bar{\bx}}_*] = \sup_{\bw\in\Wcal}|\inner{\E[\bx],\bw}|
= \sup_{\bw\in\Wcal}\left|0+\mu\sum_{i=1}^{d}\sigma_i w_i\right| =
\mu\sqrt{d} \leq \frac{1}{2},
\]
and moreover, for any $\bx$ in the support of $\Dcal_{\bsigma}$ and for any
$\bw\in\Wcal$,
\[
|\inner{\bx,\bw}| ~\leq~ |3 x_0|+\left|\sum_{i=1}^{d}x_i w_i\right|
~\leq~ |3x_0|+\mu\sum_{i=1}^{d}|w_i| ~\leq~ |3~x_0|+\mu\sqrt{d} ~\leq~ |3~x_0|+\frac{1}{2},
\]
by the fact that $\norm{\cdot}_1\leq \sqrt{d}\norm{\cdot}_2$ and our
assumption on $\mu$. $3x_0$ is normally distributed with mean zero and
variance $9*\frac{1}{36}=1/4$, from which it is easily verified that
$\Pr\left(\sup_{\bw}|\inner{\bx,\bw}|>z\right)\leq 2\exp(-z^2/2)$ for all
$z\geq 1$ as required.

We now start the proof. For any fixed $\Dcal_{\bsigma}$, recall that
$\bar{\bx}=\E_{\bx\sim\Dcal_{\bsigma}}[\bx]=(0,\mu\sigma_1,\ldots,\mu\sigma_d)$,
and define
\[
\bw^*=\left(0,-\frac{\sigma_1}{\sqrt{d}},\ldots,-\frac{\sigma_d}{\sqrt{d}}\right)=
-\frac{1}{\mu\sqrt{d}}\bar{\bx}.
\]
It is easily verified that $\bw^*$ is a minimizer of $\inner{\bar{\bx},\bw}$
over $\Wcal$. Therefore,
\begin{align*}
  &\E\left[\inner{\bar{\bx},\hat{\bw}}-\min_{\bw\in\Wcal}\inner{\bar{\bx},\bw}\right]
  ~=~\E[\inner{\bar{\bx},\hat{\bw}-\bw^*}]
  ~=~-\mu\sqrt{d}~\E[\inner{\bw^*,\hat{\bw}-\bw^*}]\\
  &=~ \mu\sqrt{d}~\E[\norm{\bw^*}_2^2-\inner{\bw^*,\hat{\bw}}]
  ~=~ \frac{\mu\sqrt{d}}{2}\E[\norm{\bw^*}_2^2+\norm{\bw^*}_2^2-2\inner{\bw^*,\hat{\bw}}].
\end{align*}
Since $\sum_{i=1}^{d}\hat{w}_i^2\leq 1=\norm{\bw^*}_2^2$ and $w^*_0=0$, we
can lower bound the above by
\begin{align}
&\frac{\mu\sqrt{d}}{2}\E\left[\sum_{i=1}^{d}(w^*)_i^2+\sum_{i=1}^{d}\hat{w}_i^2-
2\sum_{i=1}^{d}w^*_i \hat{w}_i\right]
~=~
\frac{\mu\sqrt{d}}{2}\E\left[\sum_{i=1}^{d}\left(w^*_i-\hat{w}_i\right)^2\right]\notag\\
&=
\frac{\mu\sqrt{d}}{2}\sum_{i=1}^{d}\E\left[\left(\frac{\sigma_i}{\sqrt{d}}+\hat{w}_i\right)^2\right]
~\geq~
\frac{\mu\sqrt{d}}{2}\sum_{i=1}^{d}\left(\frac{1}{\sqrt{d}}\right)^2\Pr(\sigma_i \hat{w}_i\geq 0)\notag\\
&=\frac{\mu}{2\sqrt{d}}\sum_{i=1}^{d}\Pr(\sigma_i\hat{w}_i\geq 0),\label{eq:prlow}
\end{align}
where $\Pr(\cdot)$ is the probability with respect to the joint randomness of
$\bsigma$ and $\hat{\bw}$. Since each $\sigma_i$ is uniformly distributed on
$\{-1,+1\}$, we can lower bound each probability term as follows:
\begin{align*}
  \Pr(\sigma_i\hat{w}_i\geq 0) &= \frac{1}{2}\left(\Pr(\hat{w}_i\geq 0|\sigma_i=1)+\Pr(\hat{w}_i\leq 0|\sigma_i=-1)\right)\\
  &\geq \frac{1}{2}\left(\Pr(\hat{w}_i>0|\sigma_i=1)+1-\Pr(\hat{w}_i>0|\sigma_i=-1)\right)\\
  &= \frac{1}{2}\left(1-\left(\Pr(\hat{w}_i>0|\sigma_i=-1)-\Pr(\hat{w}_i>0|\sigma_i=1)\right)\right)\\
  &\geq \frac{1}{2}\left(1-\left|\Pr(\hat{w}_i>0|\sigma_i=-1)-\Pr(\hat{w}_i>0|\sigma_i=1)\right|\right).
\end{align*}
Plugging this back to \eqref{eq:prlow}, we get the lower bound
\begin{align*}
  &\frac{\mu}{4\sqrt{d}}\sum_{i=1}^{d}\left(1-\left|\Pr(\hat{w}_i>0|\sigma_i=-1)-\Pr(\hat{w}_i>0|\sigma_i=1)\right|\right)\\
  &=~
  \frac{\mu\sqrt{d}}{4}\left(1-\frac{1}{d}\sum_{i=1}^{d}\left|\Pr(\hat{w}_i>0|\sigma_i=-1)-\Pr(\hat{w}_i>0|\sigma_i=1)\right|\right).
\end{align*}
To continue, recall that we assume the player is deterministic, in which case
$\bw=\bw(\bv)$ is a function of the sequence of losses $\bv=(v_1,\ldots,v_T)$
observed by the player over $T$ rounds. Letting $p$ denote the probability
density function over $\bv$, and $\mathbf{1}_A$ denote the indicator of the
event $A$, we can rewrite the above as
\begin{align}
  &\frac{\mu\sqrt{d}}{4}\left(1-\frac{1}{d}\sum_{i=1}^{d}\left|\int_{\bv}\mathbf{1}_{w_i(\bv)>0}p(\bv|\sigma_i=-1)d\bv
  -\int_{\bv}\mathbf{1}_{w_i>0}p(\bv|\sigma_i=1)d\bv\right|\right) \notag\\
  &=
  \frac{\mu\sqrt{d}}{4}\left(1-\frac{1}{d}\sum_{i=1}^{d}\left|\int_{\bv}\mathbf{1}_{w_i(\bv)>0}\left(p(\bv|\sigma_i=-1)
  -p(\bv|\sigma_i=1)\right)d\bv\right|\right)\notag \\
  &\geq
  \frac{\mu\sqrt{d}}{4}\left(1-\frac{1}{d}\sum_{i=1}^{d}\int_{\bv}\left|p(\bv|\sigma_i=-1)
  -p(\bv|\sigma_i=1)\right|d\bv\right).\label{eq:prprepin}
\end{align}
Each integral represents the total variation distance between the densities
$p(\cdot|\sigma_i=1)$ and $p(\cdot|\sigma_i=-1)$. By Pinsker's inequality, it
can be upper bounded as follows:
\begin{equation}\label{eq:pins}
\int_{\bv}\left|p(\bv|\sigma_i=-1)
  -p(\bv|\sigma_i=1)\right|d\bv ~\leq~
\sqrt{2~D_{kl}(p(\bv|\sigma_i=-1)||p(\bv|\sigma_i=1))},
\end{equation}
where $D_{kl}(p||q)$ is the Kullback-Leibler (KL) divergence between $p$ and
$q$. By the chain rule, this can be upper bounded in turn by
\begin{equation}\label{eq:dkl}
\sqrt{2~\sum_{t=1}^{T}D_{kl}(p(v_t|\sigma_i=-1,v_1,\ldots,v_{t-1})
~||~p(v_t|\sigma_i=1,v_1,\ldots,v_{t-1}))}.
\end{equation}
Since the player is deterministic, any values $v_1,\ldots,v_{t-1}$ determine
the point $\bw_t$ that the player will choose on round $t$. Thus, the
distribution of $v_t=\inner{\bx_t,\bw_t}$ depends on the distribution of
$\bx_t$, as determined by $\Dcal_{\bsigma}$. By definition of
$\Dcal_{\bsigma}$:
\begin{itemize}
\item Under the condition $\sigma_i=1$, $v_t=\inner{\bx_t,\bw_t}$ has a
    Gaussian distribution with mean $\mu w_{t,i}+\sum_{j\in \{1\ldots
    d\}\setminus\{i\}}(\mu\sigma_j w_{t,j})$, and variance
    $\frac{1}{36}w_{t,0}^2 \geq \frac{1}{36}$ (note that here we crucially
    use the domain assumptions on $\Wcal$, which imply $w_{t,0}\geq 1$).
\item Under the condition $\sigma_i=-1$, $v_t=\inner{\bx_t,\bw_t}$ has a
    Gaussian distribution with mean $-\mu w_{t,i}+\sum_{j\{1\ldots
    d\}\setminus\{i\}}(\mu\sigma_j w_{t,j})$, and variance
    $\frac{1}{36}w_{t,0}^2 \geq \frac{1}{36}$.
\end{itemize}
By a standard result on the KL divergence of two Gaussian distributions, we
therefore have that \eqref{eq:dkl} is at most
\[
\sqrt{2~\sum_{t=1}^{T}\E\left[72 \mu^2 w_{t,i}^2|\sigma_i=-1\right]}.
\]
Recalling that this constitutes an upper bound on the left hand side of
\eqref{eq:pins}, we get
\begin{equation}\label{eq:pinsfinal}
  \int_{\bv}\left|p(\bv|\sigma_i=-1)
  -p(\bv|\sigma_i=1)\right|d\bv ~\leq~
  \sqrt{2~\sum_{t=1}^{T}\E\left[72 \mu^2 w_{t,i}^2|\sigma_i=-1\right]}.
\end{equation}
To get rid of the conditioning on $\sigma_i=-1$, note that the total
variation is symmetric, so again using Pinsker's inequality, we can get a
variant of \eqref{eq:pins} of the form
\[
\int_{\bv}\left|p(\bv|\sigma_i=-1)
  -p(\bv|\sigma_i=1)\right|d\bv ~\leq~
\sqrt{2~D_{kl}(p(\bv|\sigma_i=1)||p(\bv|\sigma_i=-1))}.
\]
Using the same derivation as above, this leads to the upper bound
\begin{equation}\label{eq:pinsfinal2}
  \int_{\bv}\left|p(\bv|\sigma_i=-1)
  -p(\bv|\sigma_i=1)\right|d\bv ~\leq~
  \sqrt{2~\sum_{t=1}^{T}\E\left[72 \mu^2 w_{t,i}^2|\sigma_i=1\right]}.
\end{equation}
Combining \eqref{eq:pinsfinal} and \eqref{eq:pinsfinal2}, and using the
elementary inequality $\min\{\sqrt{a},\sqrt{b}\}\leq \sqrt{\frac{a+b}{2}}$
and the fact that $\sigma_i$ is uniformly distributed on $\{-1,+1\}$, we get
\begin{align}
  \int_{\bv}\left|p(\bv|\sigma_i=-1)
  -p(\bv|\sigma_i=1)\right|d\bv
  &\leq
  \min\left\{\sqrt{2~\sum_{t=1}^{T}\E\left[72 \mu^2 w_{t,i}^2|\sigma_i=-1\right]},
  \sqrt{2~\sum_{t=1}^{T}\E\left[72 \mu^2 w_{t,i}^2|\sigma_i=1\right]}\right\}\notag\\
  &\leq
  \sqrt{144\mu^2\sum_{t=1}^{T}\frac{1}{2}\left(\E\left[w_{t,i}^2|\sigma_i=-1\right]
  +\E\left[w_{t,i}^2|\sigma_i=1\right]\right)}\notag\\
  &= \mu\sqrt{144\sum_{t=1}^{T}\E[w_{t,i}^2]}~.\label{eq:pinsdone}
\end{align}
Plugging this back into \eqref{eq:prprepin}, we get an expected error lower
bound of
\[
\frac{\mu\sqrt{d}}{4}\left(1-\frac{1}{d}\sum_{i=1}^{d}\mu\sqrt{144\sum_{t=1}^{T}\E[w_{t,i}^2]}\right),
\]
which by Jensen's inequality is at least
\[
\frac{\mu\sqrt{d}}{4}\left(1-\mu\sqrt{\frac{144}{d}\sum_{i=1}^{d}\sum_{t=1}^{T}\E[w_{t,i}^2]}\right)
~=~
\frac{\mu\sqrt{d}}{4}\left(1-\mu\sqrt{\frac{144}{d}\sum_{t=1}^{T}\E\left[\sum_{i=1}^{d}w_{t,i}^2\right]}\right).
\]
By definition of $\Wcal$, we have that $\sum_{i=1}^{d}w_{t,i}^2 \leq 1$
always, so we can lower bound this by
\begin{equation}\label{eq:finallowbound}
\frac{\mu\sqrt{d}}{4}\left(1-\mu\sqrt{\frac{144 T}{d}}\right).
\end{equation}
Let us now consider two cases:
\begin{itemize}
  \item If $\frac{1}{\sqrt{d}}> \sqrt{\frac{d}{144T}}$, then we choose
      $\mu=\frac{1}{2}\sqrt{\frac{d}{144T}}$ (which ensures the condition
      $\mu\leq \frac{1}{2\sqrt{d}}$), and \eqref{eq:finallowbound} equals
      $\frac{\mu\sqrt{d}}{4}*\frac{1}{2}=\frac{d}{16\sqrt{144T}}$.
  \item If $\frac{1}{\sqrt{d}}\leq \sqrt{\frac{d}{144T}}$, then we choose
      $\mu=\frac{1}{2\sqrt{d}}$, and \eqref{eq:finallowbound} is lower
      bounded by $\frac{\mu\sqrt{d}}{4}*\frac{1}{2}=\frac{1}{16}$.
\end{itemize}
Therefore, by choosing $\mu$ appropriately, we can lower bound
\eqref{eq:finallowbound} by $\min\left\{\frac{1}{16},\frac{d}{16\sqrt{144
T}}\right\}\geq 0.005\min\left\{1,\frac{d}{\sqrt{T}}\right\}$. Since this
lower bound holds for a domain in $\reals^{d+1}$ for any $d>0$, we get a
$0.005\min\left\{1,\frac{d-1}{\sqrt{T}}\right\}$ lower bound for a domain in
$\reals^d$ for any $d>1$, as required.

\subsection{Proof of \thmref{thm:regretorigin}}

As in the proof of \thmref{thm:unorigin}, to simplify notation, we assume the
game takes place in $\reals^{d+1}$ for some $d>0$, so the domain is
$\Wcal=[-1,+1]\times \{\bw\in\reals^d: \norm{\bw}_2\leq 1\}$. We denote the
first coordinate as coordinate $0$, and the other coordinates as
$1,2,\ldots,d$.

The proof uses a somewhat different loss vector distribution than that of
\thmref{thm:unorigin}, designed to force the algorithm to choose points far
from and to one side of the origin to ensure small regret.

By Yao's minimax principle, it is sufficient to provide a randomized strategy
to choose a loss vector distribution $\Dcal$, such that for any deterministic
player, the expected error is as defined in the theorem. In particular, we
use the following strategy:
\begin{itemize}
  \item Choose $\bsigma \in \{-1,+1\}^d$ uniformly at random.
  \item Use the distribution $\Dcal_{\bsigma}$ over loss vectors $\bx$,
      defined as follows: $(x_1,\ldots,x_d)$ has a Gaussian distribution
      $\Ncal\left(\mu \bsigma,\frac{1}{16d}I\right)$ (where $\mu\leq
      \frac{1}{4\sqrt{d}}$ is a parameter to be chosen later), and $x_0$ is
      chosen independently according to a Gaussian distribution
      $\Ncal\left(-\frac{1}{4},\frac{1}{16}\right)$. Note that this is
      different than the construction in the proof of
      \thmref{thm:unorigin}, and the distribution of $x_0$ is not
      zero-mean. The idea is that $w_{t,0}$ will have to be close to $1$
      most of the time to get low regret, hence the points queried are far
      from and to one side of the origin.
\end{itemize}
First, let us verify that any $\Dcal_{\bsigma}$ is a valid distribution. We
have
\[
\E[\norm{\bar{\bx}}_*] = \sup_{\bw\in\Wcal}|\inner{\E[\bx],\bw}|
= \sup_{\bw\in\Wcal}\left|\frac{1}{4}+\mu\sum_{i=1}^{d}\sigma_i w_i\right| =
\frac{1}{4}+\mu\sqrt{d} \leq \frac{1}{4}+\frac{1}{4} < 1.
\]
Moreover, for any $\bx$ in the support of $\Dcal_{\bsigma}$ and for any
$\bw\in\Wcal$, by the Cauchy-Schwartz inequality,
\[
|\inner{\bx,\bw}| ~\leq~ \left|x_0\right|+\left|\sum_{i=1}^{d}x_i w_i\right|
~\leq~ \left|x_0\right|+\sqrt{\sum_{i=1}^{d}x_i^2}\sqrt{\sum_{i=1}^{d}w_i^2}
~\leq~ \left|x_0\right|+\sqrt{\sum_{i=1}^{d}x_i^2}.
\]
$x_0$ is normally distributed with mean $-1/4$ and variance $1/16$, and each
$x_i$ is independently normally distributed with mean at most $1/4\sqrt{d}$
and variance $1/16d$, from which it can be verified using Gaussian tail
bounds that
\[
\Pr\left(\sup_{\bw}|\inner{\bx,\bw}|>z\right)
~\leq~ \Pr\left(\left|x_0\right|+\sqrt{\sum_{i=1}^{d}x_i^2}>z\right)
~\leq~ 2\exp(-z^2/2)
\]
for all $z\geq 1$ as required.

We now start the proof. Define the scalar and $d$-dimensional vectors
\[
\hat{w}_0 = \frac{1}{T}\sum_{t=1}^{T}w_{t,0}~~,~~
\hat{\bw} = \frac{1}{T}\sum_{t=1}^{T}(w_{t,1},w_{t,2},\ldots,w_{t,d})~~,~~
\bw^* = -\frac{1}{\sqrt{d}}\bsigma.
\]
For any fixed $\Dcal_{\bsigma}$, recall that
$\bar{\bx}=\E_{\bx\sim\Dcal_{\bsigma}}[\bx]=\left(-\frac{1}{4},\mu\sigma_1,\ldots,\mu\sigma_d\right)
=-\left(\frac{1}{4},\mu\sqrt{d}w^*_1,\ldots,\mu\sqrt{d}w^*_d\right)$. It is
therefore easily verified that $\left(1,w^*_1,\ldots,w^*_d\right)$ is a
minimizer of $\inner{\bar{\bx},\bw}$ over $\Wcal$. Recalling the formulation
of regret in \eqref{eq:stochregret}, we can write it as
\begin{align}
  &T~\E\left[\inner{\bar{\bx},\frac{1}{T}\sum_{t=1}^{T}\bw_t}-\min_{\bw\in\Wcal}\inner{\bar{\bx},\bw}\right]
  ~=~T~\E\left[\frac{1}{4}\left(1-\hat{w}_0\right)-\inner{\mu\sqrt{d}\bw^*,\hat{\bw}}+\inner{\mu\sqrt{d}\bw^*,\bw^*}\right]\notag\\
  &=~T~\E\left[\frac{1}{4}\left(1-\hat{w}_0\right)+\mu\sqrt{d}~\inner{\bw^*,\bw^*-\hat{\bw}}\right]
  ~=~ T~\E\left[\frac{1}{4}\left(1-\hat{w}_0\right)+\mu\sqrt{d}\left(\norm{\bw^*}_2^2-\inner{\bw^*,\hat{\bw}}\right)\right]\notag\\
  &=~ \frac{T}{4}\E\left[1-\hat{w}_0\right]+\frac{\mu\sqrt{d}T}{2}\E[\norm{\bw^*}_2^2+\norm{\bw^*}_2^2-2\inner{\bw^*,\hat{\bw}}].\label{eq:2prlow}
\end{align}
By definition of the domain, each $(w_{t,1},\ldots,w_{t,d})$ has norm at most
$1$, and therefore $\hat{\bw}$, which is their average, also has norm at most
$1$. Therefore, $\sum_{i=1}^{d}\hat{w}_i^2\leq 1=\norm{\bw^*}_2^2$, so we can
lower bound the second expectation above by
\begin{align*}
&\E\left[\sum_{i=1}^{d}(w^*_i)^2+\sum_{i=1}^{d}\hat{w}_i^2-2\sum_{i=1}^{d}w^*_i \hat{w}_i\right]
~=~\E\left[\sum_{i=1}^{d}\left(w^*_i-\hat{w}_i\right)^2\right]
~=~
\sum_{i=1}^{d}\E\left[\left(\frac{\sigma_i}{\sqrt{d}}+\hat{w}_i\right)^2\right]\notag\\
&~\geq~
\sum_{i=1}^{d}\left(\frac{1}{\sqrt{d}}\right)^2\Pr(\sigma_i \hat{w}_i\geq 0)
~=~\frac{1}{d}\sum_{i=1}^{d}\Pr(\sigma_i\hat{w}_i\geq 0),
\end{align*}
where $\Pr(\cdot)$ is the probability with respect to the joint randomness of
$\bsigma$ and $\hat{\bw}$. Plugging this back into \eqref{eq:2prlow}, we get
a regret lower bound of the form
\[
  \frac{T}{4}\E\left[1-\hat{w}_0\right]+\frac{\mu T}{2\sqrt{d}}\sum_{i=1}^{d}\Pr(\sigma_i\hat{w}_i\geq 0).
\]
These probabilities can now be lower bounded as in the proof of
\thmref{thm:unorigin} (see the derivation following \eqref{eq:prlow}), which
imply that
\begin{align}
&\frac{T}{4}\E\left[1-\hat{w}_0\right]+\frac{\mu T}{2\sqrt{d}}\sum_{i=1}^{d}\Pr(\sigma_i\hat{w}_i\geq 0)\notag\\
&~~~\geq
\frac{T}{4}\E\left[1-\hat{w}_0\right]+
\frac{\mu\sqrt{d}T}{4}\left(1-\frac{1}{d}\sum_{i=1}^{d}\int_{\bv}|p(\bv|\sigma_i=-1)-p(\bv|\sigma_i=1)|d\bv\right),\label{eq:2prepins}
\end{align}
where $\bv=(v_1,\ldots,v_T)$ are the sequence of losses observed by the
player over the $T$ rounds, and $p$ is the probability density function over
$\bv$. As in the proof of \thmref{thm:unorigin}, we can now use Pinsker's
inequality and the chain rule to upper the integral as follows:
\begin{equation}\label{eq:2pins}
\int_{\bv}|p(\bv|\sigma_i=-1)-p(\bv|\sigma_i=1)|d\bv
~\leq~
  \sqrt{2\sum_{t=1}^{T}D_{kl}\left(p(v_t|\sigma_i=1,v_1\ldots v_{t-1})\middle|\middle|p(v_t|\sigma_i=-1,v_1\ldots v_{t-1})\right)},
\end{equation}
where $D_{kl}$ is the KL divergence. Since the player is deterministic, any
values $v_1,\ldots,v_{t-1}$ determine the point $\bw_t$ that the player will
choose on round $t$. Thus, the distribution of $v_t=\inner{\bx_t,\bw_t}$
depends on the distribution of $\bx_t$, as determined by $\Dcal_{\bsigma}$.
By definition of $\Dcal_{\bsigma}$:
\begin{itemize}
\item Under the condition $\sigma_i=1$, $v_t=\inner{\bx_t,\bw_t}$ has a
    Gaussian distribution with mean\\ $-\frac{1}{4}w_{t,0}+\mu
    w_{t,i}+\sum_{j\in \{1\ldots d\}\setminus\{i\}}(\mu\sigma_j w_{t,j})$,
    and variance
    $\frac{1}{16}w_{t,0}^2+\sum_{i=1}^{d}\frac{1}{16d}w_{t,i}^2$.
\item Under the condition $\sigma_i=-1$, $v_t=\inner{\bx_t,\bw_t}$ has a
    Gaussian distribution with mean\\ $-\frac{1}{4}w_{t,0}-\mu
    w_{t,i}+\sum_{j\{1\ldots d\}\setminus\{i\}}(\mu\sigma_j w_{t,j})$, and
    variance $\frac{1}{16}w_{t,0}^2+\sum_{i=1}^{d}\frac{1}{16d}w_{t,i}^2$.
\end{itemize}
By a standard result on the KL divergence of two Gaussian distributions, we
therefore have that \eqref{eq:2pins} is at most
\begin{equation}\label{eq:2pins1}
  \sqrt{2\sum_{t=1}^{T}\E\left[\frac{32 \mu^2 w_{t,i}^2}{w_{t,0}^2+\sum_{i=1}^{d}\frac{1}{d}w_{t,i}^2}\middle|\sigma_i=1\right]}.
\end{equation}
As in the proof of \thmref{thm:unorigin}, we can also upper bound
\eqref{eq:2pins} using the reverse order of probabilities:
\[
\int_{\bv}|p(\bv|\sigma_i=-1)-p(\bv|\sigma_i=1)|d\bv
~\leq~
  \sqrt{2\sum_{t=1}^{T}D_{kl}\left(p(v_t|\sigma_i=-1,v_1\ldots v_{t-1})\middle|\middle|p(v_t|\sigma_i=1,v_1\ldots v_{t-1})\right)},
\]
which leads to \eqref{eq:2pins} being upper bounded by
\begin{equation}\label{eq:2pins2}
  \sqrt{2\sum_{t=1}^{T}\E\left[\frac{32 \mu^2 w_{t,i}^2}{w_{t,0}^2+\sum_{i=1}^{d}\frac{1}{d}w_{t,i}^2}\middle|\sigma_i=-1\right]}.
\end{equation}
Combining \eqref{eq:2pins1} and \eqref{eq:2pins2} and using the elementary
inequality $\min\{\sqrt{a},\sqrt{b}\}\leq \sqrt{\frac{a+b}{2}}$ and the fact
that $\sigma_i$ is uniformly distributed on $\{-1,+1\}$, we get
\begin{align}
  &\int_{\bv}\left|p(\bv|\sigma_i=-1)
  -p(\bv|\sigma_i=1)\right|d\bv\notag\\
  &\leq
  \min\left\{\sqrt{2~\sum_{t=1}^{T}\E\left[\frac{32 \mu^2 w_{t,i}^2}{w_{t,0}^2+\sum_{i=1}^{d}\frac{1}{d}w_{t,i}^2}\middle|\sigma_i=-1\right]},
  \sqrt{2~\sum_{t=1}^{T}\E\left[\frac{32 \mu^2 w_{t,i}^2}{w_{t,0}^2+\sum_{i=1}^{d}\frac{1}{d}w_{t,i}^2}\middle|\sigma_i=1\right]}\right\}\notag\\
  &\leq
  \sqrt{64\mu^2\sum_{t=1}^{T}\frac{1}{2}\left(\E\left[\frac{w_{t,i}^2}{w_{t,0}^2+\sum_{i=1}^{d}\frac{1}{d}w_{t,i}^2}\middle|\sigma_i=-1\right]
  +\E\left[\frac{w_{t,i}^2}{w_{t,0}^2+\sum_{i=1}^{d}\frac{1}{d}w_{t,i}^2}\middle|\sigma_i=1\right]\right)}\notag\\
  &= 8\mu\sqrt{\sum_{t=1}^{T}\E\left[\frac{w_{t,i}^2}{w_{t,0}^2+\sum_{i=1}^{d}\frac{1}{d}w_{t,i}^2}\right]}~.\label{eq:2pinsdone}
\end{align}
Plugging this back into \eqref{eq:2prepins}, we get an expected error lower
bound of
\[
\frac{T}{4}\E\left[1-\hat{w}_0\right]+
\frac{\mu\sqrt{d}T}{4}\left(1-\frac{1}{d}\sum_{i=1}^{d}8\mu\sqrt{\sum_{t=1}^{T}\E\left[\frac{w_{t,i}^2}{w_{t,0}^2+\sum_{i=1}^{d}\frac{1}{d}w_{t,i}^2}\right]}\right),
\]
which by Jensen's inequality is at least
\begin{align*}
&\frac{T}{4}\E\left[1-\hat{w}_0\right]+
\frac{\mu\sqrt{d}T}{4}\left(1-8\mu\sqrt{\frac{1}{d}\sum_{i=1}^{d}\sum_{t=1}^{T}\E\left[\frac{w_{t,i}^2}{w_{t,0}^2+\sum_{i=1}^{d}\frac{1}{d}w_{t,i}^2}\right]}\right)\\
&~~~=~
\frac{T}{4}\E\left[1-\hat{w}_0\right]+
\frac{\mu\sqrt{d}T}{4}\left(1-8\mu\sqrt{\frac{1}{d}\sum_{t=1}^{T}\E\left[\frac{\sum_{i=1}^{d}w_{t,i}^2}{w_{t,0}^2+\sum_{i=1}^{d}\frac{1}{d}w_{t,i}^2}\right]}\right)\\
&~~~=~
\frac{T}{4}\E\left[1-\hat{w}_0\right]+
\frac{\mu\sqrt{d}T}{4}\left(1-8\mu\sqrt{\frac{1}{d}\sum_{t=1}^{T}\E\left[\frac{1}{\frac{w_{t,0}^2}{\sum_{i=1}^{d}w_{t,i}^2}+\frac{1}{d}}\right]}\right).
\end{align*}
By definition of $\Wcal$, we have that $\sum_{i=1}^{d}w_{t,i}^2 \leq 1$
always. Using this and recalling that
$\hat{w}_0=\frac{1}{T}\sum_{t=1}^{T}w_{t,0}$, we can lower bound the above by
\begin{equation}\label{eq:2finallowbound}
\frac{1}{4}\sum_{t=1}^{T}\E\left[1-w_{t,0}\right]+
\frac{\mu\sqrt{d}T}{4}\left(1-8\mu\sqrt{\frac{1}{d}\sum_{t=1}^{T}\E\left[\frac{1}{w_{t,0}^2+\frac{1}{d}}\right]}\right).
\end{equation}
The trick now is to argue that not matter what are the values of $w_{t,0}$,
this lower bound will be large: Either $w_{t,0}$ will tend to be bounded away
from $1$, and then \eqref{eq:2finallowbound} will be large due to the first
term; Or that $w_{t,0}$ will tend to be very close to $1$, but then
\eqref{eq:2finallowbound} will be large due to the second term. To make this
precise, we use the following technical lemma:
\begin{lemma}\label{lem:dw}
  For any integer $d\geq 1$, and for any $w\in [-1,1]$, it holds that $
  \frac{1}{w^2+\frac{1}{d}}\leq d(1-|w|)+1$.
\end{lemma}
\begin{proof}
  Since both sides of the inequality are the same for $w$ and $-w$, we can
  assume without loss of generality, that $w\in [0,1]$, and prove that
  $\frac{1}{w^2+\frac{1}{d}}\leq d(1-w)+1$ over this domain.
  
  By algebraic manipulations, this inequality is equivalent to the assertion
  that $(d+1)w^2-dw^3-w\geq -\frac{1}{d}$, or equivalently,
  \[
  w(1-w)(dw-1)\geq -\frac{1}{d}.
  \]
  Considering the domain of $w$, the left hand side is non-positive only when $w\in
  [0,1/d]$, because of the third term, but in that regime $(dw-1)\in [-1,0]$ and $(1-w)\in [0,1]$, and
  therefore
  \[
  w(1-w)(dw-1) \geq w*1*(-1) = -w \geq -\frac{1}{d}
  \]
  as required.
\end{proof}
Applying this inequality, we can lower bound \eqref{eq:2finallowbound} by
\begin{align*}
&\frac{1}{4}\sum_{t=1}^{T}\E\left[1-|w_{t,0}|\right]+
\frac{\mu\sqrt{d}T}{4}\left(1-8\mu\sqrt{\frac{1}{d}\sum_{t=1}^{T}\E\left[d(1-|w_{t,0}|)+1)\right]}\right)\\
&=~
\frac{1}{4}\sum_{t=1}^{T}\E\left[1-|w_{t,0}|\right]+
\frac{\mu\sqrt{d}T}{4}\left(1-8\mu\sqrt{\frac{T}{d}+\sum_{t=1}^{T}\E\left[1-|w_{t,0}|\right]}\right)\\
&\geq~
\min_{z\in [0,T]}\left(
\frac{z}{4}+
\frac{\mu\sqrt{d}T}{4}\left(1-8\mu\sqrt{\frac{T}{d}+z}\right)\right),
\end{align*}
Recalling that $\mu$ is a free parameter, let us choose
$\mu=\frac{1}{16}\sqrt{\frac{d}{T}}$ (which satisfies the assumption $\mu\leq
\frac{1}{4\sqrt{d}}$, since we assume $T\geq d^4/16$). We therefore get
\begin{equation}\label{eq:nu}
\frac{1}{4}\min_{z\in [0,T]}\left(
z+\frac{1}{16}d\sqrt{T}\left(1-\frac{1}{2}\sqrt{1+\frac{d}{T}z}\right)\right).
\end{equation}
Note that this expression is convex with respect to $z$, and by
differentiating, has an extremal point at
\[
z=2^{-12}d^3-\frac{T}{d}.
\]
Since we assume $T\geq d^4/16$, this point is non-positive, and therefore the
minimum in \eqref{eq:nu} is attained at $z=0$, where it equals
\[
\frac{1}{64}d\sqrt{T}\left(1-\frac{1}{2}\right)= \frac{1}{128}d\sqrt{T}.
\]
Since this lower bound holds for a domain in $\reals^{d+1}$ for any $d>0$, we
get a $\frac{1}{128}(d-1)\sqrt{T}$ lower bound for a domain in $\reals^d$ for
any $d>1$, as required.

\subsection{Proof of \thmref{thm:simplexunorigin}}

By Yao's minimax principle, it is sufficient to provide a randomized strategy
to choose a loss vector distribution $\Dcal$, such that for any deterministic
player, the expected error is as defined in the theorem. In particular, we
use the following strategy:
\begin{itemize}
  \item Choose $J \in \{1,\ldots,d\}$ uniformly at random.
  \item Use the distribution $\Dcal_{J}$ over loss vectors $\bx$, defined
      as $\Ncal(-\mu\be_J,\frac{1}{4}I)$ (where $\mu\leq \frac{1}{2}$ is a
      parameter to be chosen later).
\end{itemize}
First, let us verify that any $\Dcal_{J}$ is a valid distribution. We have
\[
\E[\norm{\bar{\bx}}_*] = \sup_{\bw\in\Wcal}|\inner{\E[\bx],\bw}|
= \sup_{\bw\in\Wcal}|\mu w_j|\leq \mu \leq \frac{1}{2},
\]
and moreover, for any $\bx$ in the support of $\Dcal_{\bsigma}$ and for any
$\bw\in\Wcal$, $\inner{\bx,\bw}$ is Gaussian with mean $\mu w_{j}\in
\left[-\frac{1}{2},\frac{1}{2}\right]$ and variance
$\frac{1}{4}\sum_{i=1}^{d}w_i^2\leq \frac{1}{4}\sum_{i=1}^{d}|w_i|\leq
\frac{1}{4}$, from which it is easily verified that
$\Pr\left(\sup_{\bw}|\inner{\bx,\bw}|>z\right)\leq 2\exp(-z^2/2)$ for all
$z\geq 1$ as required.

We now start the proof. For any fixed $\Dcal_{\bsigma}$, recall that
$\bar{\bx}=-\mu \be_j$, and define $\bw^*=\be_j$. It is easily verified that
$\bw^*$ is a minimizer of $\inner{\bar{\bx},\bw}$ over $\Wcal$. Also, let
$\E_j$ denote expectation assuming that $J=j$, and recall that $J$ is chosen
uniformly at random from $\{1,\ldots,d\}$. Therefore,
\[
  \E\left[\inner{\bar{\bx},\hat{\bw}}-\min_{\bw\in\Wcal}\inner{\bar{\bx},\bw}\right]
  ~=~
  \frac{1}{d}\sum_{j=1}^{d}\E_j\left[\mu-\mu\E[\hat{w}_j\right]
  ~=~
  \mu-\frac{\mu}{d}\sum_{j=1}^{d}\E_j[\hat{w}_j].
\]
Define the reference distribution $\Dcal_0=\Ncal(\mathbf{0},\frac{1}{2}I)$
over loss vectors, which has zero-mean, and let $\E_0$ denote expectation
assuming the loss vectors are chosen from that distribution. Then we can
lower bound the above by
\begin{align}
\mu-\frac{\mu}{d}\sum_{j=1}^{d}\left(\E_0[\hat{w}_j]+\left|\E_0[\hat{w}_j]-\E_j[\hat{w}_j]\right|\right)&
\geq~
\mu-\frac{\mu}{d}\left(\E_0\left[\sum_{j=1}^{d}|\hat{w}_j|\right]+\sum_{j=1}^{d}\left|\E_0[\hat{w}_j]-\E_j[\hat{w}_j]\right|\right)\notag\\
&\geq~
\left(1-\frac{1}{d}\right)\mu-\frac{\mu}{d}\sum_{j=1}^{d}\left|\E_0[\hat{w}_j]-\E_j[\hat{w}_j]\right|,\label{eq:3lowexp}
\end{align}
where in the last inequality we used the fact that $\bw$ has $1$-norm at most
$1$. Now, recall that we assume that the player is deterministic, hence
$\hat{\bw}=\hat{\bw}(\bv)$ is a deterministic function of the sequence of
loss observations $\bv=(v_1,\ldots,v_T)$ made by the player over $T$ rounds.
Letting $p_0(\bv)$ and $p_j(\bv)$ denote the density functions with respect
to $\Dcal_0,\Dcal_j$ respectively, we can upper bound the expectation
differences in \eqref{eq:3lowexp} as follows:
\begin{align*}
  \sum_{j=1}^{d}\left|\E_0[\hat{w}_j]-\E_j[\hat{w_j}]\right|
  &=~
  \sum_{j=1}^{d}\left|\int_{\bv}\hat{w}_j(\bv)(p_0(\bv)-p_j(\bv))d\bv\right|\\
  &\leq~
  \sum_{j=1}^{d}\int_{\bv}|\hat{w}_j(\bv)|\left|p_0(\bv)-p_j(\bv)\right|d\bv\\
  &\leq~
  \sum_{j=1}^{d}\int_{\bv}\left|p_0(\bv)-p_j(\bv)\right|d\bv.
\end{align*}
By Pinsker's inequality, and the chain rule, this can be upper bounded by
\begin{equation}\label{eq:pins}
\sum_{j=1}^{d}\sqrt{2\sum_{t=1}^{T}D_{kl}\left(p_0(v_t|v_1\ldots v_{t-1})\middle|\middle|
p_j(v_t|v_1\ldots v_{t-1})\right)},
\end{equation}
where $D_{kl}(p||q)$ is the Kullback-Leibler (KL) divergence between $p$ and
$q$. Since the player is deterministic, any values $v_1,\ldots,v_{t-1}$
determine the point $\bw_t$ that the player will choose on round $t$. Thus,
the distribution of $v_t=\inner{\bx_t,\bw_t}$ depends on the distribution of
$\bx_t$:
\begin{itemize}
\item Under $p_j$ (corresponding to $\Dcal_j$), $v_t=\inner{\bx_t,\bw_t}$
    has a Gaussian distribution with mean $-\mu w_{t,j}$ and variance
    $\frac{1}{4}\sum_{i=1}^{d} w_{t,i}^2=\frac{1}{4}\norm{\bw_t}_2$.
\item Under $p_0$ (corresponding to $\Dcal_0$), $v_t=\inner{\bx_t,\bw_t}$
    has a Gaussian distribution with mean $0$ and variance
    $\frac{1}{4}\norm{\bw_t}_2$.
\end{itemize}
By a standard result on the KL divergence of two Gaussian distributions, we
therefore have that \eqref{eq:pins} is at most
\[
\sum_{j=1}^{d}\sqrt{2~\sum_{t=1}^{T}\E_0\left[\frac{2\mu^2 w_{t,j}^2}{\norm{\bw_t}_2}\right]}.
\]
By Jensen's inequality, this is at most
\begin{align*}
  &d\frac{1}{d}\sum_{j=1}^{d}\sqrt{2~\sum_{t=1}^{T}\E_0\left[\frac{2\mu^2 w_{t,j}^2}{\norm{\bw_t}_2}\right]}
  ~\leq~ d\sqrt{\frac{2}{d}\sum_{j=1}^{d}\sum_{t=1}^{T}\E_0\left[\frac{2\mu^2 w_{t,j}^2}{\norm{\bw_t}_2}\right]}\\
  &=d\sqrt{\frac{2}{d}\sum_{t=1}^{T}\E_0\left[\frac{2\sum_{j=1}^{d}\mu^2 w_{t,j}^2}{\norm{\bw_t}_2}\right]}
  ~=~d\sqrt{\frac{2}{d}\sum_{t=1}^{T}(2\mu^2)}
  ~=~ 2\mu\sqrt{dT}.
\end{align*}
Plugging this back into \eqref{eq:3lowexp}, we get an error lower bound of
\begin{equation}\label{eq:3end}
\left(1-\frac{1}{d}\right)\mu-2\mu^2\sqrt{\frac{T}{d}}
~=~ \mu\left(1-\frac{1}{d}-\mu\sqrt{\frac{T}{d}}\right)
~\geq~
\mu\left(\frac{1}{2}-\mu\sqrt{\frac{T}{d}}\right),
\end{equation}
where we used the fact that $d>1$. Let us now consider two cases:
\begin{itemize}
  \item If $T\geq \frac{d}{4}$, then we choose
      $\mu=\frac{1}{4}\sqrt{\frac{d}{T}}$ (which ensures the condition
      $\mu\leq \frac{1}{2}$), and \eqref{eq:3end} equals
      $\frac{1}{4}\sqrt{\frac{d}{T}}\left(\frac{1}{2}-\frac{1}{4}\right)=
      \frac{1}{16}\sqrt{\frac{d}{T}}$.
  \item If $T< \frac{d}{4}$, then we choose $\mu=\frac{1}{2}$, and
      \eqref{eq:3end} is lower bounded by
      $\frac{1}{2}\left(\frac{1}{2}-\frac{1}{2}\sqrt{\frac{1}{4}}\right)=\frac{1}{8}$.
\end{itemize}
Therefore, by choosing $\mu$ appropriately, we can lower bound
\eqref{eq:finallowbound} by
$\min\left\{\frac{1}{8},\frac{1}{16}\sqrt{\frac{d}{T}}\right\}\geq
\frac{1}{16}\min\left\{1,\sqrt{\frac{d}{T}}\right\}$ as required.

\subsection{Proof of \thmref{thm:hypercubeorigin}}

As in the proof of \thmref{thm:regretorigin}, to simplify notation, we assume
the game takes place in $\reals^{d+1}$ for some $d>0$, so the domain is
$\Wcal=[-1,+1]\times [-1,+1]^d$. We denote the first coordinate as coordinate
$0$, and the other coordinates as $1,2,\ldots,d$.

The proof uses a similar loss vector distribution as that of
\thmref{thm:regretorigin}, but with a different scaling to ensure a valid
distribution.

By Yao's minimax principle, it is sufficient to provide a randomized strategy
to choose a loss vector distribution $\Dcal$, such that for any deterministic
player, the expected error is as defined in the theorem. In particular, we
use the following strategy:
\begin{itemize}
  \item Choose $\bsigma \in \{-1,+1\}^d$ uniformly at random.
  \item Use the distribution $\Dcal_{\bsigma}$ over loss vectors $\bx$,
      defined as follows: $(x_1,\ldots,x_d)$ has a Gaussian distribution
      $\Ncal\left(\mu \bsigma,\frac{1}{16d^2}I\right)$ (where $\mu\leq
      \frac{1}{4d}$ is a parameter to be chosen later), and $x_0$ is chosen
      independently according to a Gaussian distribution
      $\Ncal\left(-\frac{1}{4},\frac{1}{16}\right)$.
\end{itemize}
First, let us verify that any $\Dcal_{\bsigma}$ is a valid distribution. We
have
\[
\E[\norm{\bar{\bx}}_*] = \sup_{\bw\in\Wcal}|\inner{\E[\bx],\bw}|
= \sup_{\bw\in\Wcal}\left|\frac{1}{4}+\mu\sum_{i=1}^{d}\sigma_i w_i\right| =
\frac{1}{4}+\mu d \leq \frac{1}{4}+\frac{1}{4} < 1.
\]
Moreover, for any $\bx$ in the support of $\Dcal_{\bsigma}$ and for any
$\bw\in\Wcal$, by H\"{o}lder's inequality,
\[
|\inner{\bx,\bw}| ~\leq~ \left|x_0\right|+\sum_{i=1}^{d}|x_i|.
\]
$x_0$ is normally distributed with mean $-1/4$ and variance $1/16$, and each
$x_i$ is independently normally distributed with mean at most $1/4d$ and
variance $1/16d^2$, from which it can be verified using Gaussian tail bounds
that
\[
\Pr\left(\sup_{\bw}|\inner{\bx,\bw}|>z\right)
~\leq~ \Pr\left(|x_0|+\sum_{i=1}^{d}|x_i|>z\right)
~\leq~ 2\exp(-z^2/2)
\]
for all $z\geq 1$ as required.

We now start the proof. Define the scalar and $d$-dimensional vectors
\[
\hat{w}_0 = \frac{1}{T}\sum_{t=1}^{T}w_{t,0}~~,~~
\hat{\bw} = \frac{1}{T}\sum_{t=1}^{T}(w_{t,1},w_{t,2},\ldots,w_{t,d})~~,~~
\bw^* = -\bsigma.
\]
For any fixed $\Dcal_{\bsigma}$, recall that
$\bar{\bx}=\E_{\bx\sim\Dcal_{\bsigma}}[\bx]=\left(-\frac{1}{4},\mu\sigma_1,\ldots,\mu\sigma_d\right)$.
It is therefore easily verified that $(1,\sigma_1,\ldots,\sigma_d)$ is a
minimizer of $\inner{\bar{\bx},\bw}$ over $\Wcal$. Recalling the formulation
of regret in \eqref{eq:stochregret}, we can write it as
\begin{align}
  &T~\E\left[\inner{\bar{\bx},\frac{1}{T}\sum_{t=1}^{T}\bw_t}-\min_{\bw\in\Wcal}\inner{\bar{\bx},\bw}\right]
  ~=~T~\E\left[\frac{1}{4}\left(1-\hat{w}_0\right)+\mu d-\mu\sum_{i=1}^{d}(-\sigma_i)\hat{w}_i\right]\notag\\
  &=~\frac{T}{4}\E[1-\hat{w}_0]+\mu T\left(d-\sum_{i=1}^{d}\E\left[(-\sigma_i)\hat{w}_i\right]\right)\label{eq:4prlow}
\end{align}
Recalling that each $\sigma_i$ is uniformly distributed on $\{-1,+1\}$, the
sum of expectations can be upper bounded as follows:
\begin{align*}
  \sum_{i=1}^{d}\E\left[(-\sigma_i)\hat{w}_i\right]
  &= \frac{1}{2}\sum_{i=1}^{d}\left(\E[\hat{w}_i|\sigma_i=-1]+\E[-\hat{w}_i|\sigma_i=1]\right)\\
  &= \frac{1}{2}\sum_{i=1}^{d}\left(\E[\hat{w}_i|\sigma_i=-1]-\E[\hat{w}_i|\sigma_i=1]\right).
\end{align*}
Since the player's strategy is assumed to be deterministic, then
$\hat{w}_i=\hat{w}_i(\bv)$ is a deterministic function of the sequence of
losses $\bv=(v_1,\ldots,v_T)$ observed by the player over the $T$ rounds.
Letting $p(\cdot)$ denote the probability density function of $\bv$, we can
rewrite the above as
\begin{align}
  &\frac{1}{2}\sum_{i=1}^{d}\int_{\bv}\hat{w}_i(\bv)(p(\bv|\sigma_i=-1)-p(\bv|\sigma_i=1))d\bv
  ~\leq~ \frac{1}{2}\sum_{i=1}^{d}\int_{\bv}|\hat{w}_i(\bv)||(p(\bv|\sigma_i=-1)-p(\bv|\sigma_i=1))|d\bv\notag\\
  &\leq~ \frac{1}{2}\sum_{i=1}^{d}\int_{\bv}|(p(\bv|\sigma_i=-1)-p(\bv|\sigma_i=1))|d\bv,\label{eq:4prlow0}
\end{align}
where we used the fact that $\hat{w}_i\in [-1,+1]$ by the domain assumptions.
We can now use Pinsker's inequality and the chain rule to upper the integral
as follows:
\begin{equation}\label{eq:4pins}
\int_{\bv}|p(\bv|\sigma_i=-1)-p(\bv|\sigma_i=1)|d\bv
~\leq~
  \sqrt{2\sum_{t=1}^{T}D_{kl}\left(p(v_t|\sigma_i=1,v_1\ldots v_{t-1})\middle|\middle|p(v_t|\sigma_i=-1,v_1\ldots v_{t-1})\right)},
\end{equation}
where $D_{kl}$ is the KL divergence. Since the player is deterministic, any
values $v_1,\ldots,v_{t-1}$ determine the point $\bw_t$ that the player will
choose on round $t$. Thus, the distribution of $v_t=\inner{\bx_t,\bw_t}$
depends on the distribution of $\bx_t$, as determined by $\Dcal_{\bsigma}$.
By definition of $\Dcal_{\bsigma}$:
\begin{itemize}
\item Under the condition $\sigma_i=1$, $v_t=\inner{\bx_t,\bw_t}$ has a
    Gaussian distribution with mean\\ $-\frac{1}{4}w_{t,0}+\mu
    w_{t,i}+\sum_{j\in \{1\ldots d\}\setminus\{i\}}(\mu\sigma_j w_{t,j})$,
    and variance
    $\frac{1}{16}w_{t,0}^2+\sum_{i=1}^{d}\frac{1}{16d^2}w_{t,i}^2$.
\item Under the condition $\sigma_i=-1$, $v_t=\inner{\bx_t,\bw_t}$ has a
    Gaussian distribution with mean\\ $-\frac{1}{4}w_{t,0}-\mu
    w_{t,i}+\sum_{j\{1\ldots d\}\setminus\{i\}}(\mu\sigma_j w_{t,j})$, and
    variance
    $\frac{1}{16}w_{t,0}^2+\sum_{i=1}^{d}\frac{1}{16d^2}w_{t,i}^2$.
\end{itemize}
By a standard result on the KL divergence of two Gaussian distributions, we
therefore have that \eqref{eq:4pins} is at most
\begin{equation}\label{eq:4pins1}
  \sqrt{2\sum_{t=1}^{T}\E\left[\frac{32 \mu^2 w_{t,i}^2}{w_{t,0}^2+\sum_{i=1}^{d}\frac{1}{d^2}w_{t,i}^2}\middle|\sigma_i=1\right]}.
\end{equation}
Applying Pinsker's inequality on the reverse order of probabilities in
\eqref{eq:4pins}, we can also upper \eqref{eq:4pins} by
\[
\int_{\bv}|p(\bv|\sigma_i=-1)-p(\bv|\sigma_i=1)|d\bv
~\leq~
  \sqrt{2\sum_{t=1}^{T}D_{kl}\left(p(v_t|\sigma_i=-1,v_1\ldots v_{t-1})\middle|\middle|p(v_t|\sigma_i=1,v_1\ldots v_{t-1})\right)},
\]
which leads to \eqref{eq:4pins} being upper bounded by
\begin{equation}\label{eq:4pins2}
  \sqrt{2\sum_{t=1}^{T}\E\left[\frac{32 \mu^2 w_{t,i}^2}{w_{t,0}^2+\sum_{i=1}^{d}\frac{1}{d^2}w_{t,i}^2}\middle|\sigma_i=-1\right]}.
\end{equation}
Combining \eqref{eq:4pins1} and \eqref{eq:4pins2} and using the elementary
inequality $\min\{\sqrt{a},\sqrt{b}\}\leq \sqrt{\frac{a+b}{2}}$ and the fact
that $\sigma_i$ is uniformly distributed on $\{-1,+1\}$, we get
\begin{align*}
  &\int_{\bv}\left|p(\bv|\sigma_i=-1)
  -p(\bv|\sigma_i=1)\right|d\bv\notag\\
  &\leq
  \min\left\{\sqrt{2~\sum_{t=1}^{T}\E\left[\frac{32 \mu^2 w_{t,i}^2}{w_{t,0}^2+\sum_{i=1}^{d}\frac{1}{d^2}w_{t,i}^2}\middle|\sigma_i=-1\right]},
  \sqrt{2~\sum_{t=1}^{T}\E\left[\frac{32 \mu^2 w_{t,i}^2}{w_{t,0}^2+\sum_{i=1}^{d}\frac{1}{d^2}w_{t,i}^2}\middle|\sigma_i=1\right]}\right\}\notag\\
  &\leq
  \sqrt{64\mu^2\sum_{t=1}^{T}\frac{1}{2}\left(\E\left[\frac{w_{t,i}^2}{w_{t,0}^2+\sum_{i=1}^{d}\frac{1}{d^2}w_{t,i}^2}\middle|\sigma_i=-1\right]
  +\E\left[\frac{w_{t,i}^2}{w_{t,0}^2+\sum_{i=1}^{d}\frac{1}{d^2}w_{t,i}^2}\middle|\sigma_i=1\right]\right)}\notag\\
  &= 8\mu\sqrt{\sum_{t=1}^{T}\E\left[\frac{w_{t,i}^2}{w_{t,0}^2+\sum_{i=1}^{d}\frac{1}{d^2}w_{t,i}^2}\right]}.
\end{align*}
Plugging this back into \eqref{eq:4prlow0}, and that in turn into
\eqref{eq:4prlow}, we get an expected error lower bound of
\[
\frac{T}{4}\E\left[1-\hat{w}_0\right]+\mu T\left(d-4\mu\sum_{i=1}^{d}\sqrt{\sum_{t=1}^{T}\E\left[\frac{w_{t,i}^2}{w_{t,0}^2+\sum_{i=1}^{d}\frac{1}{d^2}w_{t,i}^2}\right]}\right),
\]
which by Jensen's inequality is at least
\begin{align*}
&\frac{T}{4}\E\left[1-\hat{w}_0\right]+
\mu T\left(d-4\mu d\sqrt{\frac{1}{d}\sum_{i=1}^{d}\sum_{t=1}^{T}\E\left[\frac{w_{t,i}^2}{w_{t,0}^2+\sum_{i=1}^{d}\frac{1}{d^2}w_{t,i}^2}\right]}\right)\\
&~~~=~
\frac{T}{4}\E\left[1-\hat{w}_0\right]+
\mu T d\left(1-4\mu \sqrt{\frac{1}{d}\sum_{t=1}^{T}\E\left[\frac{\sum_{i=1}^{d}w_{t,i}^2}{w_{t,0}^2+\sum_{i=1}^{d}\frac{1}{d^2}w_{t,i}^2}\right]}\right)\\
&~~~=~
\frac{T}{4}\E\left[1-\hat{w}_0\right]+
\mu T d\left(1-4\mu\sqrt{\frac{1}{d}\sum_{t=1}^{T}\E\left[\frac{1}{\frac{w_{t,0}^2}{\sum_{i=1}^{d}w_{t,i}^2}+\frac{1}{d^2}}\right]}\right).
\end{align*}
By definition of $\Wcal$, we have that $\sum_{i=1}^{d}w_{t,i}^2 \leq
\sum_{i=1}^{d}|w_{t,i}|\leq d$ always. Using this and recalling that
$\hat{w}_0=\frac{1}{T}\sum_{t=1}^{T}w_{t,0}$, we can lower bound the above by
\[
\frac{1}{4}\sum_{t=1}^{T}\E\left[1-w_{t,0}\right]+
\mu T d\left(1-4\mu\sqrt{\sum_{t=1}^{T}\E\left[\frac{1}{w_{t,0}^2+\frac{1}{d}}\right]}\right).
\]

Using \lemref{lem:dw} from the proof of \thmref{thm:regretorigin}, this
expression can be lower bounded by
\begin{align*}
&\frac{1}{4}\sum_{t=1}^{T}\E\left[1-|w_{t,0}|\right]+
\mu T d\left(1-4\mu\sqrt{\sum_{t=1}^{T}\E\left[d(1-|w_{t,0}|)+1)\right]}\right)\\
&=~
\frac{1}{4}\sum_{t=1}^{T}\E\left[1-|w_{t,0}|\right]+
\mu T d\left(1-4\mu\sqrt{T+d\sum_{t=1}^{T}\E\left[1-|w_{t,0}|\right]}\right)\\
&\geq~
\min_{z\in [0,T]}\left(
\frac{z}{4}+
\mu T d\left(1-4\mu\sqrt{T+dz}\right)\right),
\end{align*}
Recalling that $\mu$ is a free parameter, let us choose
$\mu=\frac{1}{8}\sqrt{\frac{1}{T}}$ (which satisfies the assumption $\mu\leq
\frac{1}{4d}$, since we assume $T\geq d^4/4$). We therefore get
\begin{equation}\label{eq:nu}
\min_{z\in [0,T]}\left(
\frac{z}{4}+\frac{1}{8}d\sqrt{T}\left(1-\frac{1}{2}\sqrt{1+\frac{d}{T}z}\right)\right).
\end{equation}
Note that this expression is convex with respect to $z$, and by
differentiating, has an extremal point at
\[
z=\frac{1}{64}d^3-\frac{T}{d}.
\]
Since we assume $T\geq d^4/4$, this point is non-positive, and therefore the
minimum in \eqref{eq:nu} is attained at $z=0$, where it equals
\[
\frac{1}{8}d\sqrt{T}\left(1-\frac{1}{2}\right)= \frac{1}{16}d\sqrt{T}.
\]
Since this lower bound holds for a domain in $\reals^{d+1}$ for any $d>0$, we
get a $\frac{1}{16}(d-1)\sqrt{T}$ lower bound for a domain in $\reals^d$ for
any $d>1$, as required.

\subsubsection*{Acknowledgments}
This research was supported by an Israel Science Foundation grant 425/13 and
an FP7 Marie Curie CIG grant. We thank S\'{e}bastien Bubeck for several
illuminating discussions.

\bibliographystyle{plain}
\bibliography{mybib}

\appendix

\section{From Sub-Gaussian to Bounded Distributions}\label{app:bounded}

As discussed in \secref{sec:setting}, our results use sub-Gaussian
distributions rather than the more standard bounded distribution assumption.
In this appendix, we explain why this is really without any loss of
generality, and that lower bounds for the sub-Gaussian setting can be readily
converted to the bounded setting, at the cost of a $\sqrt{\log(T)}$ factor.

In particular, suppose there is a sub-Gaussian distribution $\Dcal$ over the
cost vectors, for which any player's strategy incurs expected error/regret at
least $R$. The trick is to consider a ``shrinked'' distribution, which simply
re-scales all cost vectors by $\Theta(1/\sqrt{\log(T)})$. Then the regret
will be $\Omega(R/\sqrt{\log(T)})$, and because the distribution is
sub-Gaussian, then with very high probability all the $T$ cost vectors
$\bx_1,\ldots,\bx_T$ satisfy $\norm{\bx_t}_{*}\leq 1$. This means that even
if we modify the distribution to force it to be bounded, the expected error
regret will still be $\Omega(R/\sqrt{\log(T)})$. We note that a similar
technique was used implicitly in \cite{dekel2014bandits}.

A more formal result can be stated as follows:
\begin{theorem}\label{thm:transform}
  Suppose that there exists a distribution $\Dcal$ satisfying $\norm{\E_{\bx\sim\Dcal}[\bx]}_{*}=1$ and $\Pr_{\bx\sim\Dcal}(\norm{\bx}_{*}>z)\leq 2\exp(-z^2/2)$ for all $z\geq
  1$, such that for any player's strategy, the expected error/regret is at least $R$. Assuming $T>1$, there exists a
  distribution $\Dcal'$ satisfying $\Pr_{\bx\sim \Dcal'}(\norm{\bx}_*\leq 1)=1$,
  for which the expected error/regret of any player is at least
  $c\left(\frac{R}{\sqrt{\log(T)}}-\frac{1}{T^4}\right)$ for some universal constant $c$.
\end{theorem}
Since the regret $R$ is virtually always at least $\Omega(1/T)$, this means
that we get the same error/regret up to constant and $\sqrt{\log(T)}$
factors. The $1/T^4$ term can be replaced by $1/T^s$ for arbitrarily large
$s$, at the cost of affecting the $c$ constant.
\begin{proof}[Proof Sketch]
  Let $p$ be a parameter to be determined later. Given the distribution $\Dcal$, we
  algorithmically define $\Dcal'$ and an auxiliary distribution $\hat{\Dcal}$ as
  follows:
  \begin{itemize}
    \item $\hat{\Dcal}$ samples $\bx$ according to $\Dcal$ and returns
        $\frac{1}{p\sqrt{\log(T)}}\bx$. Let
        $\bar{\bx}_{\hat{\Dcal}}=\E_{\hat{\Dcal}}[\bx]$.
    \item $\Dcal'$ samples $\bx$ according to $\hat{\Dcal}$, and returns
        $\bx$ if $\norm{\bx}_{*}\leq \frac{1}{2}$, and otherwise returns
        $\E_{\hat{\Dcal}}\left[\bx\middle|\norm{\bx}_{*}\geq
        \frac{1}{2}\right]$. Let $\bar{\bx}_{\Dcal'}=\E_{\Dcal'}[\bx]$.
  \end{itemize}
  
  It is easily verified that $\bar{\bx}_{\hat{\Dcal}}=\bar{\bx}_{\Dcal'}$.
  Moreover, if $p$ is large enough, then $\Dcal'$
  always returns a vector such that $\norm{\bx}_{*}\leq 1$,
  and thus satisfies the lemma's requirement. In that case, we also have
  $\norm{\bar{\bx}_{\hat{\Dcal}}}_*=\norm{\bar{\bx}_{\Dcal'}}_*\leq 1$.
  
  It remains to show the error/regret lower bound. We sketch the proof for the expected error - the proof for the expected
  regret is identical by replacing $\hat{\bw}$ by
  $\frac{1}{T}\sum_{t=1}^{T}\bw_t$. By the lemma's assumptions, we know that
  \[
  \E_{\Dcal}\left[\inner{\bar{\bx}_{\Dcal},\hat{\bw}}-\min_{\bw\in\Wcal}\inner{\bar{\bx}_{\Dcal},\bw}\right] \geq R,
  \]
  where $\E_{\Dcal}$ signifies expectation with respect to drawing cost
  vectors from $\Dcal$, and $\bar{\bx}_{\Dcal}=\E_{\Dcal}[\bx]$. Since
  $\hat{\Dcal}$ simply scales the vectors drawn from $\Dcal$ by a fixed factor
  $\frac{1}{p\sqrt{\log(T)}}$, we have
  \[
  \E_{\hat{\Dcal}}\left[\inner{\bar{\bx}_{\hat{\Dcal}},\hat{\bw}}-\min_{\bw\in\Wcal}\inner{\bar{\bx}_{\hat{\Dcal}},\bw}\right] \geq \frac{R}{p\sqrt{\log(T)}}.
  \]
  Now, let $A$ be the event that for all cost vectors $\bx_1,\ldots,\bx_T$, we have $\norm{\bx_t}_*\leq
  \frac{1}{2}$. By the assumptions
  on $\Dcal,\hat{\Dcal}$ and a union bound, $\Pr_{\hat{\Dcal}}(\neg A)\leq 2T^{1-p^2/8}$ for all sufficiently large
  $p$. From the displayed equation above, it follows that
  \begin{align*}
  \frac{R}{p\sqrt{\log(T)}} &\leq \E_{\hat{\Dcal}}\left[\inner{\bar{\bx}_{\hat{\Dcal}},\hat{\bw}}-\min_{\bw\in\Wcal}\inner{\bar{\bx}_{\hat{\Dcal}},\bw}\right]\\
  &= {\Pr}_{\hat{\Dcal}}(A)\E_{\hat{\Dcal}}\left[\inner{\bar{\bx}_{\hat{\Dcal}},\hat{\bw}}-\min_{\bw\in\Wcal}\inner{\bar{\bx}_{\hat{\Dcal}},\bw}\middle|A\right]
  +{\Pr}_{\hat{\Dcal}}(\neg A)\E_{\hat{\Dcal}}\left[\inner{\bar{\bx}_{\hat{\Dcal}},\hat{\bw}}-\min_{\bw\in\Wcal}\inner{\bar{\bx}_{\hat{\Dcal}},\bw}\middle|\neg A\right]\\
  &\leq \E_{\hat{\Dcal}}\left[\inner{\bar{\bx}_{\hat{\Dcal}},\hat{\bw}}-\min_{\bw\in\Wcal}\inner{\bar{\bx}_{\hat{\Dcal}},\bw}\middle|A\right]
  +2T^{1-p^2/8}*2\max_{\bw\in\Wcal}|\inner{\bar{\bx}_{\hat{\Dcal}},\bw}|\\
  &\leq \E_{\hat{\Dcal}}\left[\inner{\bar{\bx}_{\hat{\Dcal}},\hat{\bw}}-\min_{\bw\in\Wcal}\inner{\bar{\bx}_{\hat{\Dcal}},\bw}\middle|A\right]
  +4T^{1-p^2/8},
  \end{align*}
  where we used the assumption that
  $\max_{\bw\in\Wcal}|\inner{\bw,\bar{\bx}_{\hat{\Dcal}}}|=\norm{\bar{\bx}_{\hat{\Dcal}}}_*\leq 1$. Switching sides and using the assumption that $\bar{\bx}_{\hat{\Dcal}}=\bar{\bx}_{\Dcal'}$, we get that
  \[
  \E_{\hat{\Dcal}}\left[\inner{\bar{\bx}_{\Dcal'},\hat{\bw}}-\min_{\bw\in\Wcal}\inner{\bar{\bx}_{\Dcal'},\bw}\middle|A\right]
  \geq \frac{R}{p\sqrt{\log(T)}}-4T^{1-p^2/8}.
  \]
  But conditioned on $A$, the distribution of $\bx_1,\ldots,\bx_T$ is the
  same under $\hat{\Dcal}$ and $\Dcal'$, and therefore
  \[
  \E_{\Dcal'}\left[\inner{\bar{\bx}_{\Dcal'},\hat{\bw}}-\min_{\bw\in\Wcal}\inner{\bar{\bx}_{\Dcal'},\bw}\middle|A'\right]
  \geq \frac{R}{p\sqrt{\log(T)}}-4T^{1-p^2/8},
  \]
  where $A'$ is the event that all cost vectors $\bx_1,\ldots,\bx_T$ were
  drawn based on $\hat{\Dcal}$ without modification.
  Since the error term is non-negative, this implies that
  \[
  \E_{\Dcal'}\left[\inner{\bar{\bx}_{\Dcal'},\hat{\bw}}-\min_{\bw\in\Wcal}\inner{\bar{\bx}_{\Dcal'},\bw}\right]
  \geq {\Pr}_{\Dcal'}(A')\E_{\Dcal'}\left[\inner{\bar{\bx}_{\Dcal'},\hat{\bw}}-\min_{\bw\in\Wcal}\inner{\bar{\bx}_{\Dcal'},\bw}\middle|A'\right]\geq {\Pr}_{\Dcal'}(A')\left(\frac{R}{p\sqrt{\log(T)}}-4T^{1-p^2/8}\right).
  \]
  Finally, we have that ${\Pr}_{\Dcal'}(A')={\Pr}_{\hat{\Dcal}}(A)\geq
  1-2T^{1-p^2/8}$, so we get
  \[
  \E_{\Dcal'}\left[\inner{\bar{\bx}_{\Dcal'},\hat{\bw}}-\min_{\bw\in\Wcal}\inner{\bar{\bx}_{\Dcal'},\bw}\right] \geq \left(1-2T^{1-p^2/8}\right)\left(\frac{R}{p\sqrt{\log(T)}}-4T^{1-p^2/8}\right).
  \]
  Picking $p$ sufficiently large, the result follows.
\end{proof}

\end{document}